\documentclass{article}

\usepackage{microtype}
\usepackage{graphicx}
\usepackage{booktabs}
\usepackage{hyperref}
\usepackage{xspace}
\usepackage{amsmath}
\usepackage{amssymb}
\usepackage{amsthm}
\usepackage{subfig}

\newtheorem{lemma}{Lemma}

\newcommand{\trainedEmb}{\texttt{T4096}\xspace}
\newcommand{\randomEmb}{\texttt{R64}\xspace}
\newcommand{\randomEmbfull}{\texttt{R4096}\xspace}
\newcommand*{\eg}{\textit{e.g.}\@\xspace}

\usepackage[accepted]{icml2020}

\begin{document}

\twocolumn[
\icmltitle{Reliable Fidelity and Diversity Metrics for Generative Models}
\icmlsetsymbol{equal}{*}

\begin{icmlauthorlist}
\icmlauthor{Muhammad Ferjad Naeem}{equal,naver,tum}
\icmlauthor{Seong Joon Oh}{equal,naver}
\icmlauthor{Youngjung Uh}{naver}
\icmlauthor{Yunjey Choi}{naver}
\icmlauthor{Jaejun Yoo}{naver,epfl}
\end{icmlauthorlist}

\icmlaffiliation{naver}{Clova AI Research, Naver Corp.}
\icmlaffiliation{tum}{Technische Universit\"at M\"unchen, Germany}
\icmlaffiliation{epfl}{\'Ecole polytechnique f\'ed\'erale de Lausanne (EPFL), Switzerland}

\icmlcorrespondingauthor{Jaejun Yoo}{jaejun.yoo88@gmail.com}

\icmlkeywords{Generative models, Evaluation, Generative adversarial networks}

\icmltitlerunning{Reliable Fidelity and Diversity Metrics for Generative Models}

\vskip 0.3in
]

\begin{NoHyper}
\printAffiliationsAndNotice{\icmlEqualContribution}
\end{NoHyper}

\begin{abstract}
Devising indicative evaluation metrics for the image generation task remains an open problem. The most widely used metric for measuring the similarity between real and generated images has been the Fr\'echet Inception Distance (FID) score. Because it does not differentiate the \textit{fidelity} and \textit{diversity} aspects of the generated images, recent papers have introduced variants of precision and recall metrics to diagnose those properties separately. In this paper, we show that even the latest version of the precision and recall metrics are not reliable yet. For example, they fail to detect the match between two identical distributions, they are not robust against outliers, and the evaluation hyperparameters are selected arbitrarily. We propose \textbf{density and coverage} metrics that solve the above issues. We analytically and experimentally show that density and coverage provide more interpretable and reliable signals for practitioners than the existing metrics. Code: { \href{https://github.com/clovaai/generative-evaluation-prdc}{github.com/clovaai/generative-evaluation-prdc}
}.

\end{abstract}

\section{Introduction}

Assessing a generative model is difficult. Unlike the evaluation of discriminative models $P(T|X)$ that is often easily done by measuring the prediction performances on a few labelled samples $(X_i,T_i)$, generative models $P(X)$ are assessed by measuring the discrepancy between the real $\{X_i\}$ and generated (fake) $\{Y_j\}$ sets of high-dimensional data points. Adding to the complexity, there are more than one way of measuring distances between two distributions each with its own pros and cons. In fact, even human judgement based measures like Mean Opinion Scores (MOS) are not ideal, as practitioners have diverse opinions on what the ``ideal'' generative model is \cite{borji2019pros}. 

Nonetheless, there must be some measurement of the quality of generative models for the progress of science. Several quantitative metrics have been proposed, albeit with their own set of trade-offs. For example, Fr\'echet Inception Distance (FID) score~\cite{fid2017}, the most popular metric in image generation tasks, has empirically exhibited good agreements with human perceptual scores. However, FID summarises the comparison of two distributions into a single number, failing to separate two important aspects of the quality of generative models: \textbf{fidelity and diversity}~\cite{pr2018}. Fidelity refers to the degree to which the generated samples resemble the real ones. Diversity, on the other hand, measures whether the generated samples cover the full variability of the real samples. 

Recent papers~\cite{pr2018, pr2019, ipr2019} have introduced precision and recall metrics as measures of fidelity and diversity, respectively. Though precision and recall metrics have introduced the important perspectives in generative model evaluation, we show that they are not ready yet for practical use. We argue that necessary conditions for useful evaluation metrics are: (1) ability to detect identical real and fake distributions, (2) robustness to outlier samples, (3) responsiveness to mode dropping, and (4) the ease of hyperparameter selection in the evaluation algorithms. Unfortunately, even the most recent version of the precision and recall metrics~\cite{ipr2019} fail to meet the requirements.

To address the practical concerns, we propose the \textbf{density and coverage} metrics. By introducing a simple yet carefully designed manifold estimation procedure, we not only make the fidelity-diversity metrics empirically reliable but also theoretically analysable. We test our metric on generative adversarial networks, one of the most successful generative models in recent years. 

We then study the embedding algorithms for evaluating image generation algorithms. Embedding is an inevitable ingredient due to the high-dimensionality of images and the lack of semantics in the RGB space. Despite the importance, the embedding pipeline has been relatively less studied in the existing literature; evaluations of generated images mostly rely on the features from an ImageNet pre-trained model~\cite{inceptionscore2016, fid2017, pr2018, pr2019, ipr2019, deng2009imagenet}. This sometimes limits the fair evaluation and provides a false sense of improvement, since the pre-trained models inevitably include the dataset bias~\cite{torralba2011unbiased,geirhos2018}. We show that such pre-trained embeddings often exhibit unexpected behaviours as the target distribution moves away from the natural image domain.  

To exclude the dataset bias, we consider using randomly initialised CNN feature extractors~\cite{ulyanov2018deep}. We compare the evaluation metrics on MNIST and sound generation tasks using the random embeddings. We observe that random embeddings provide more macroscopic views on the distributional discrepancies. In particular, random embeddings provide more sensible evaluation results when the target data distribution is significantly different from ImageNet statistics (\eg MNIST and spectrograms).

\begin{figure*}[!ht]
\centering
\setlength{\tabcolsep}{2em}
\begin{tabular}{c|c}
   \subfloat[Precision versus density.\label{fig:overview_p_vs_d}]{\includegraphics[width=0.36\linewidth]{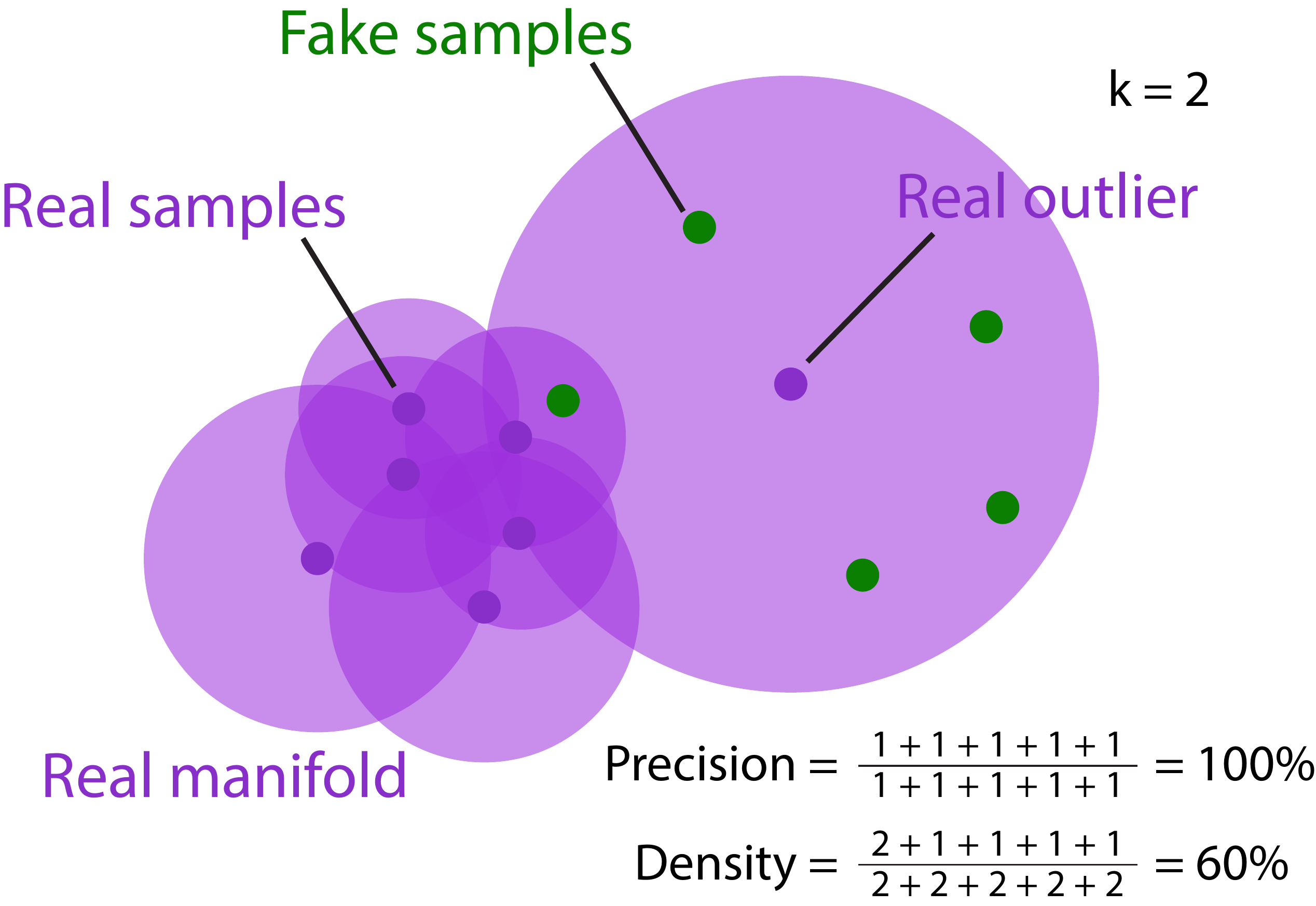}}  & \subfloat[Recall versus coverage.\label{fig:overview_r_vs_c}]{\includegraphics[width=0.45\linewidth]{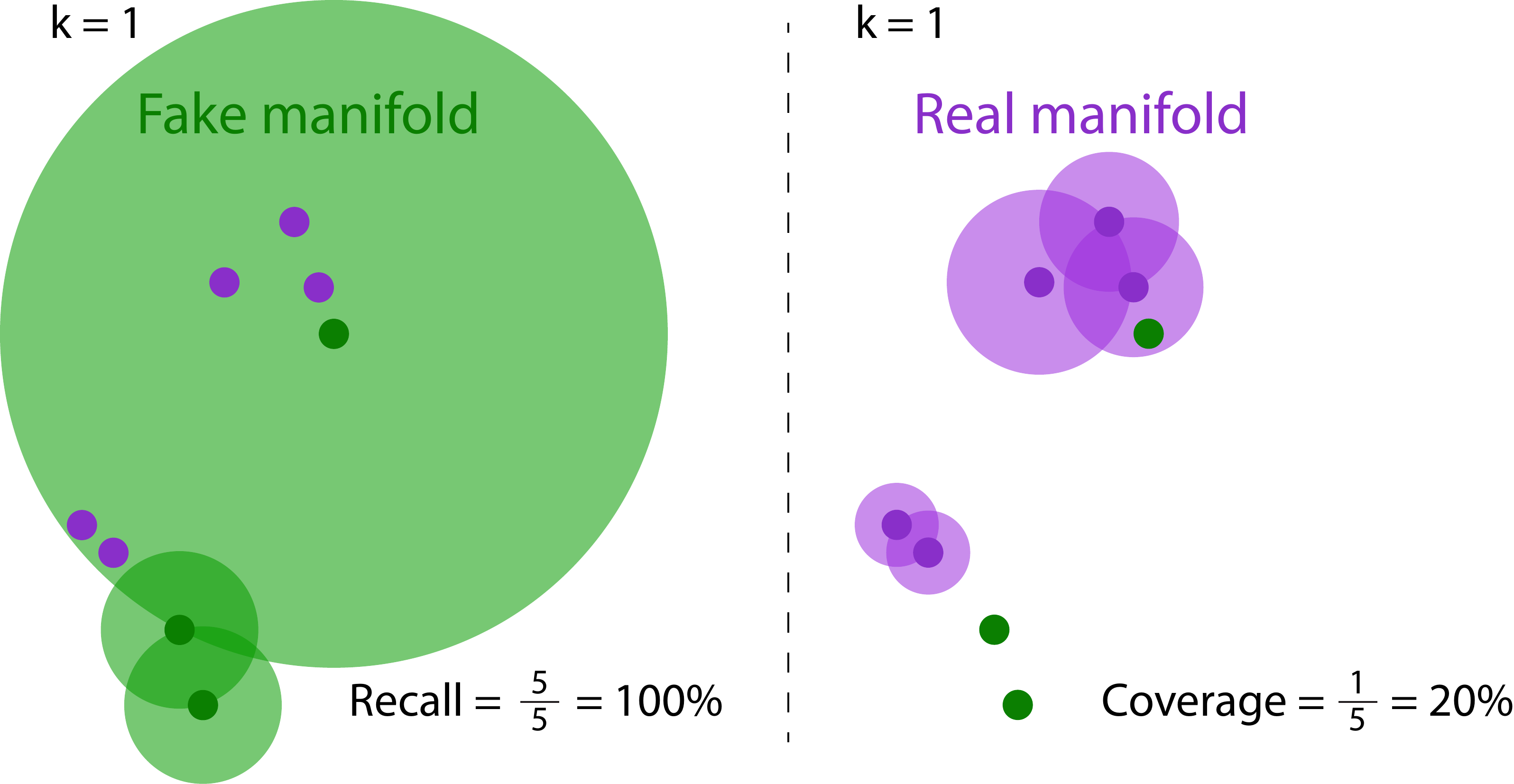}}
\end{tabular}
\caption{\small \textbf{Overview of metrics.} Two example scenarios for illustrating the advantage of density over precision and coverage over recall. Note that for recall versus coverage figure, the real and fake samples are identical across left and right.}
\label{fig:overview}
\end{figure*}

\section{Backgrounds}

Given a real distribution $P(X)$ and a generative model $Q(Y)$, we assume that we can sample $\{X_i\}$ and $\{Y_j\}$, respectively. We need an algorithm to assess how likely the sets of samples are arising from the same distribution. When the involved distribution families are tractable and the full density functions can be easily estimated, statistical testing methods or distributional distance measures (\eg Kullback-Leibler Divergence or Expected Likelihood) are viable. However, when the data $P(X)$ are complex and high-dimensional (\eg natural images), it becomes difficult to apply such measures naively~\cite{theis2016note}. Because of the difficulty, the evaluation of samples from generative models is still an actively researched topic. In this section, we provide an overview of existing approaches. We describe the most widely-used evaluation pipeline for image generative models (\S\ref{subsec:evaluation_pipeline}), and then introduce the prior works on fidelity and diversity measures (\S\ref{subsec:prior_fidelity_diversity}). For a more extensive survey, see \cite{borji2019pros}.

\subsection{Evaluation pipeline}
\label{subsec:evaluation_pipeline}

It is difficult to conduct statistical analyses over complex and high-dimensional data $X$ in their raw form (\eg images). Thus, evaluation metrics for image generative models largely follow the following stages:
(1) embed real and fake data ($\{X_i\}$ and $\{Y_j\}$) into a Euclidean space $\mathbb{R}^D$ through a non-linear mapping $f$ like CNN feature extractors, (2) construct real and fake distributions over $\mathbb{R}^D$ with the embedded samples $\{f(X_i)\}$ and $\{f(Y_j)\}$, and (3) quantify the discrepancy between the two distributions. We describe each stage in the following paragraphs.

\textbf{Embeddings.} It is often difficult to define a sensible metric over the input space. For example, $\ell_2$ distance over the image pixels $\|X_i-Y_j\|_2$ is misleading because two perceptually identical images may have great $\ell_2$ distances (one-pixel translation)~\cite{theis2016note}. To overcome this difficulty, researchers have introduced ImageNet pre-trained CNN feature extractors as the embedding function $f$ in many generative model evaluation metrics~\cite{inceptionscore2016,fid2017,pr2018,ipr2019} based on the reasoning that the $\ell_2$ distance in the feature space $\|f(X_i)-f(Y_j)\|_2$ provide sensible proxies for the human perceptual metric~\cite{zhang2018unreasonable}. Since we always use embedded samples for computing metrics, we write $X_i$ and $Y_j$ for $f(X_i)$ and $f(Y_j)$, respectively.

In this work, we also adopt ImageNet pre-trained CNNs for the embedding, but we also criticise their use when the data distribution is distinct from the ImageNet distribution (\S\ref{subsec:random_embeddings}). We suggest randomly-initialised CNN feature extractors as an alternative in such cases~\cite{ulyanov2018deep,zhang2018unreasonable}. We show that for MNIST digits or sound spectrograms with large domain gaps from ImageNet, random embeddings provide more sensible evaluation measures. 

\textbf{Building and comparing distributions.}
Given embedded samples $\{X_i\}$ and $\{Y_j\}$, many metrics conduct some form of (non-)parametric statistical estimation. \textit{Parzen window estimates}~\cite{bengio13parzen} approximate the likelihoods of the fake samples $\{Y_j\}$ by estimating the density $P(X)$ with Gaussian kernels around the real samples $\{X_i\}$. On the parametric side, \textit{Inception scores (IS)}~\cite{inceptionscore2016} estimate the multinomial distribution $P(T|Y_j)$ over the 1000 ImageNet classes for each sample image $Y_j$ and compares it against the estimated marginalised distribution $P(T)$ with the KL divergence. \textit{Fr\'echet Inception Distance (FID)}~\cite{fid2017} estimates the mean $\mu$ and covariance $\Sigma$ for $X$ and $Y$ assuming that they are multivariate Gaussians. The distance between the two Gaussians is computed by the Fr\'echet distance~\cite{dowson1982frechet}, also known as the Wasserstein-2 distance~\cite{vaserstein1969markov}. FID has been reported to generally match with human judgements~\cite{fid2017}; it has been the most popular metric for image generative models in the last couple of years~\cite{borji2019pros}.

\subsection{Fidelity and diversity}
\label{subsec:prior_fidelity_diversity}
While single-value metrics like IS and FID have led interesting advances in the field by ranking generative models, they are not ideal for diagnostic purposes. One of the most important aspects of the generative models is the trade-off between \textit{fidelity} (how realistic each input is) and \textit{diversity} (how well fake samples capture the variations in real samples). We introduce variants of the two-value metrics (precision and recall) that capture the two characteristics separately.

\textbf{Precision and recall.} \cite{pr2018} have reported the pathological case where two generative models have similar FID scores, while their qualitative fidelity and diversity results are different. For a better diagnosis of generative models, \cite{pr2018} have thus proposed the \textit{precision and recall} metrics based on the estimated supports of the real and fake distributions. Precision is defined as the portion of $Q(Y)$ that can be generated by $P(X)$; recall is symmetrically defined as the portion of $P(X)$ that can be generated by $Q(Y)$. While conceptually useful, they have multiple practical drawbacks. It assumes that the embedding space is uniformly dense, relies on the initialisation-sensitive k-means algorithm for support estimation, and produces an infinite number of values as the metric.

\textbf{Improved precision and recall.}
\cite{ipr2019} have proposed the \textit{improved precision and recall (P\&R)} that address the above drawbacks. The probability density functions are estimated via k-nearest neighbour distances, overcoming the uniform-density assumption and the reliance on the k-means. Our proposed metrics are based on P\&R; we explain the full details of P\&R here.

P\&R first constructs the ``manifold'' for $P(X)$ and $Q(Y)$ separately, the object is nearly identical to the probabilistic density function except that it does not sum to 1. Precision then measures the expected likelihood of fake samples against the real manifold and recall measures the expected likelihood of real samples against the fake manifold: %
\begin{align}
    \text{precision}&:=\frac{1}{M}\sum_{j=1}^{M}1_{Y_j\in\text{manifold}(X_1,\cdots,X_N)}
    \label{eq:precision}
    \\
    \text{recall}&:=\frac{1}{N}\sum_{i=1}^{N}1_{X_i\in\text{manifold}(Y_1,\cdots,Y_M)}
    \label{eq:recall}
\end{align}
where $N$ and $M$ are the number of real and fake samples. $1_{(\cdot)}$ is the indicator function.
Manifolds are defined
\begin{align}
    \text{manifold}(X_1,\cdots,X_N):=
    \bigcup_{i=1}^{N} B(X_i,\text{NND}_k(X_i))
    \label{eq:manifold}
\end{align}
where $B(x,r)$ is the sphere in $\mathbb{R}^D$ around $x$ with radius $r$. $\text{NND}_k(X_i)$ denotes the distance from $X_i$ to the $k^{\text{th}}$ nearest neighbour among $\{X_i\}$ excluding itself. Example computation of P\&R is shown in Figure~\ref{fig:overview}. 

There are similarities between the precision above and the Parzen window estimate~\cite{bengio13parzen}. If the manifolds are formed by \textit{superposition} instead of \textit{union} in Equation~\ref{eq:manifold} and the spheres $B(\cdot,\cdot)$ are replaced with Gaussians of fixed variances, then the manifold estimation coincides with the kernel density estimation. In this case, Equation~\ref{eq:precision} computes the expected likelihood of fake samples.

\section{Density and Coverage}
We propose novel performance measures \textbf{density and coverage (D\&C)} as practically usable measures that successfully remedy the problems with precision and recall.

\subsection{Problems with improved precision and recall}
\label{subsec:problem_precision_recall}

Practicality of the improved precision and recall (P\&R) is still compromised due to their vulnerability to outliers and computational inefficiency. Building the nearest neighbour manifolds (Equation~\ref{eq:manifold}) must be performed carefully because the spheres around each sample are not normalised according to the their radii or the relative density of samples in the neighbourhood. Consequently, the nearest neighbour manifolds generally overestimate the true manifold around the outliers, leading to undesired effects in practice. We explain the drawbacks at the conceptual level here; they will be quantified in \S\ref{sec:experiments}.

\textbf{Precision.}
We first show a pathological case for precision in Figure~\ref{fig:overview_p_vs_d}. Because of the real outlier sample, the manifold is overestimated. Generating many fake samples around the real outlier is enough to increase the precision measure. 

\textbf{Recall.}
The nearest neighbour manifold is built upon the fake samples in Equation~\ref{eq:recall}. Since models tend to generate many unrealistic yet diverse samples, the fake manifold is often an overestimation of the actual fake distribution. The pathological example is shown in Figure~\ref{fig:overview_r_vs_c}. While the fake samples are far from the modes in real samples, the recall measure gains points for real samples contained in the overestimated fake manifold. Another problem with relying on fake manifolds for the recall computation is that the manifold must be computed per model. For example, to generate the recall-vs-iteration curve for training diagnosis, the k-nearest neighbours for all fake samples must be computed ($O(kM\log M)$) for every data point.

\subsection{Density and coverage}

We remedy the issues with P\&R above and propose new metrics: density and coverage (D\&C).

\textbf{Density.}
\textit{Density} improves upon the precision metric by fixing the overestimation of the manifold around real outliers. Precision counts the binary decision of whether the fake data $Y_j$ contained in \textit{any} neighbourhood sphere $\{B(X_i,\text{NND}_k(X_i))\}_i$ of real samples (Equation~\ref{eq:precision}). Density, instead, counts \textit{how many} real-sample neighbourhood spheres contain $Y_j$ (Equation~\ref{eq:density}).
The manifold is now formed by the superposition of the neighbourhood spheres $\{B(X_i,\text{NND}_k(X_i))\}_i$, and a form of expected likelihood of fake samples is measured. In this sense, the density measure is at the midway between the precision metric and the Parzen window estimate. Density is defined as
\begin{align}
\text{density}:=&\frac{1}{kM}\sum_{j=1}^{M}\sum_{i=1}^{N}1_{Y_j\in B(X_i,\text{NND}_k(X_i))}
\label{eq:density}
\end{align}
where $k$ is for the k-nearest neighbourhoods. Through this modification, density rewards samples in regions where real samples are densely packed, relaxing the vulnerability to outliers. For example, the problem of overestimating precision (100\%) in Figure~\ref{fig:overview_p_vs_d} is resolved using the density measure (60\%). Note that unlike precision, density is not upper bounded by 1; it may be greater than 1 depending on the density of reals around the fakes.

\textbf{Coverage.}
Diversity, intuitively, shall be measured by the ratio of real samples that are covered by the fake samples. \textit{Coverage} improves upon the recall metric to better quantify this by building the nearest neighbour manifolds around the real samples, instead of the fake samples, as they have less outliers. Moreover, the manifold can only be computed per dataset, instead of per model, reducing the heavy nearest neighbour computations in recall. Coverage is defined as
\begin{align}
\text{coverage}:=&\frac{1}{N}\sum_{i=1}^{N}1_{\exists\text{ $j$ s.t. } Y_j\in B(X_i,\text{NND}_k(X_i))}.
\label{eq:coverage}
\end{align}
It measures the fraction of real samples whose neighbourhoods contain at least one fake sample. Coverage is bounded between 0 and 1.

\subsection{Analytic behaviour of density \& coverage}
\label{subsec:theory_density_converage}

The simplest sanity check for an evaluation metric is whether the metric attains the best value when the intended criteria are met. For generative models, we examine if D\&C attain 100\% performances when the real and fake distribution are identical ($P\overset{d}{=}Q$). We show that, unlike P\&R, D\&C yield an analytic expression for the expected values $\mathbb{E}[\text{density}]$ and $\mathbb{E}[\text{coverage}]$ for identical real and fake, and the values approach 100\% as the sample sizes $(N,M)$ and number of neighbourhoods $k$ increase. This analysis further leads to a systematic algorithm for selecting the hyperparameters $(k,N,M)$.

\textbf{Lack of analytic results for P\&R.}
From Equation~\ref{eq:precision}, the expected precision for identical real and fake is
\begin{align}
\mathbb{E}[\text{precision}]
&=\mathbb{P}[X_0\in\text{manifold}(X_1,\cdots,X_N)] \\
&=\mathbb{P}\left[X_0\in\bigcup_{i=1}^{N} B(X_i,\text{NND}_k(X_i))\right] \\
&=\mathbb{P}\left[\cup_{i=1}^{N} A^k_{i}\right]
\label{eq:difficult}
\end{align} 
where $A^k_i$ is the event $\{\|X_0-X_i\| < \text{NND}_k(X_i)\}$. Since $(A^k_i)_i$ are not independent with complex dependence structures, a simple expression for Equation~\ref{eq:difficult} does not exist. Same observation holds for $\mathbb{E}[\text{recall}]$.

\begin{figure*}[!t]
\newcommand{\hyperparamtablesubfloatwidth}{.195\linewidth}
\small
\setlength{\tabcolsep}{0.1em}
    \centering
    \begin{tabular}{ccc|ccc}
        \multicolumn{2}{c}{\textbf{Precision}} && \multicolumn{3}{c}{\textbf{Density}} \\
        \cline{1-2} \cline{4-6}
        \vspace{-.8em} & \\
        {\small Gaussian} & {\small FFHQ} && {\small Gaussian} & {\small FFHQ} & {\small Analytic} \\
        \includegraphics[width=\hyperparamtablesubfloatwidth]{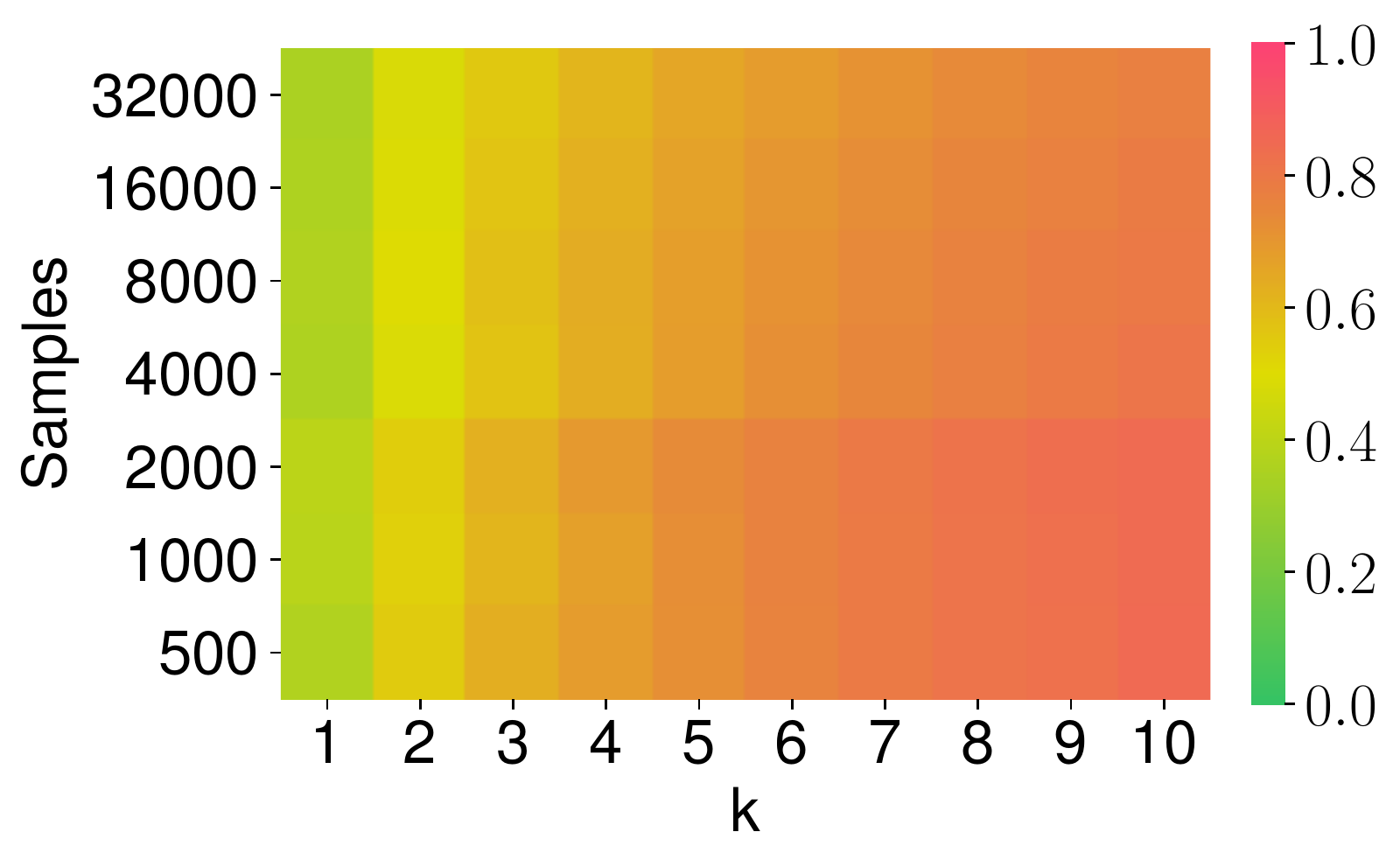} &
        \includegraphics[width=\hyperparamtablesubfloatwidth]{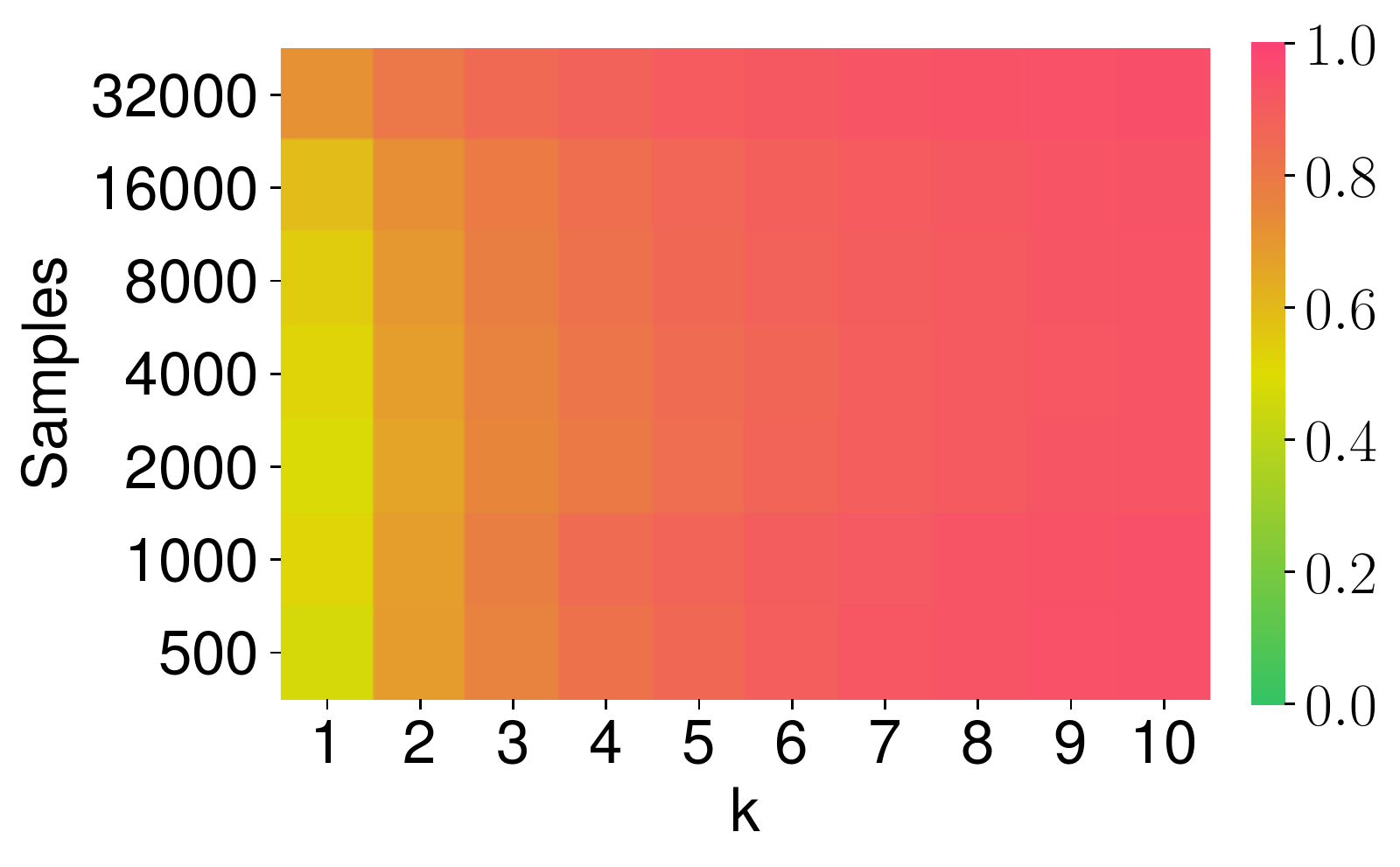} &&
        \includegraphics[width=\hyperparamtablesubfloatwidth]{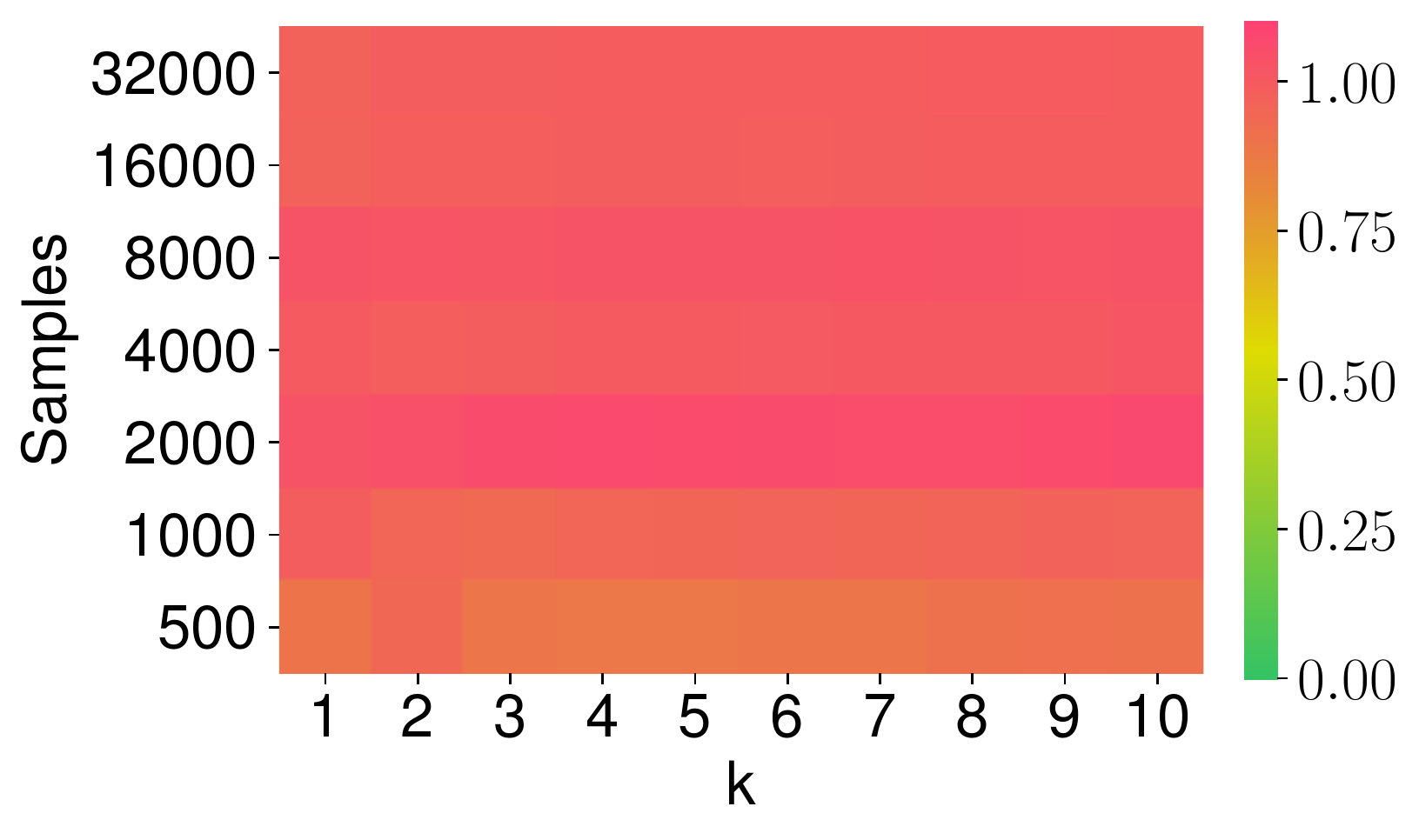} &
        \includegraphics[width=\hyperparamtablesubfloatwidth]{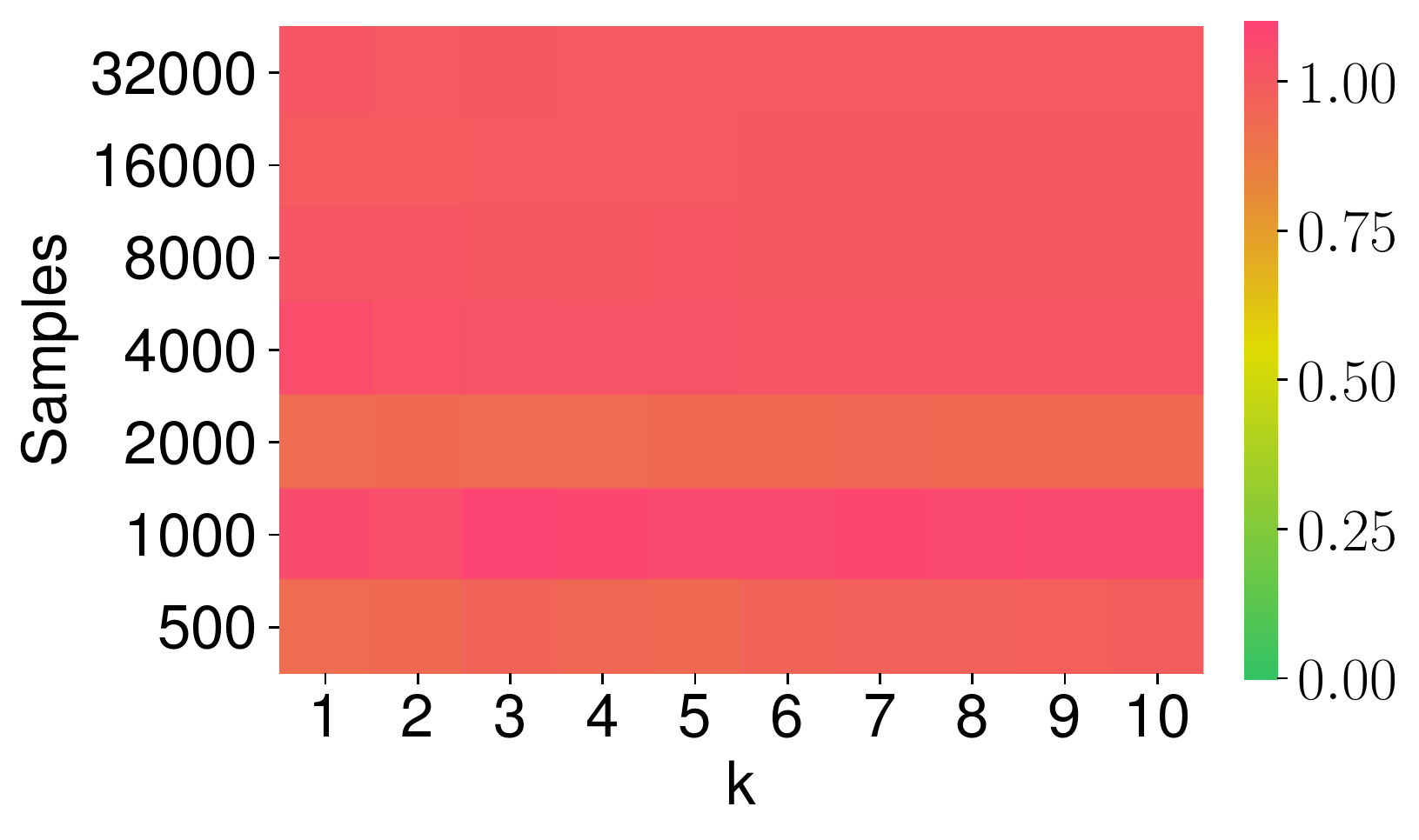} &
        \includegraphics[width=\hyperparamtablesubfloatwidth]{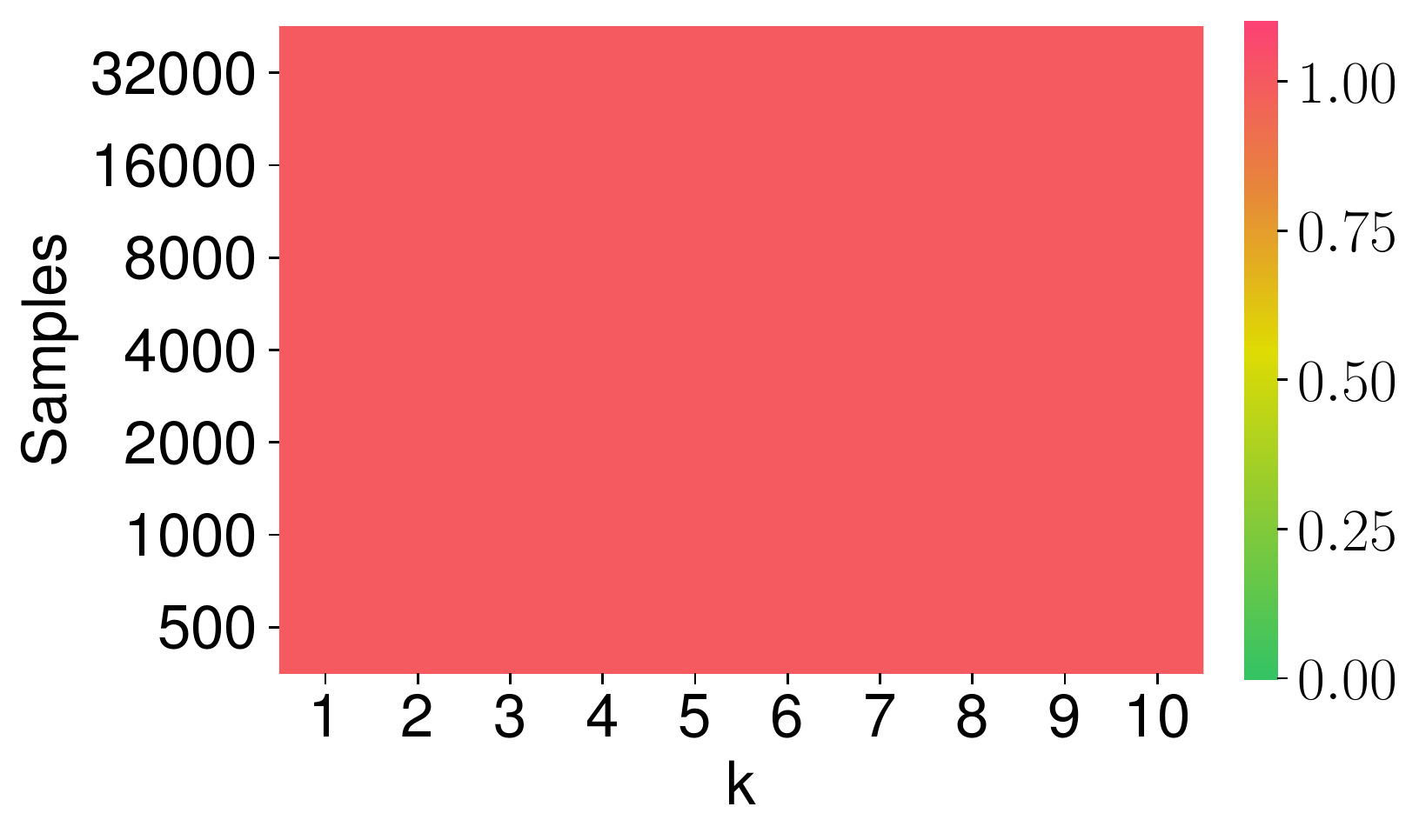} \\
        \vspace{-1.4em} & \\
        \multicolumn{2}{c}{\textbf{Recall}} && \multicolumn{3}{c}{\textbf{Coverage}} \\
        \cline{1-2} \cline{4-6}
        \vspace{-.8em} & \\
        {\small Gaussian} & {\small FFHQ} && {\small Gaussian} & {\small FFHQ} & {\small Analytic} \\
        \includegraphics[width=\hyperparamtablesubfloatwidth]{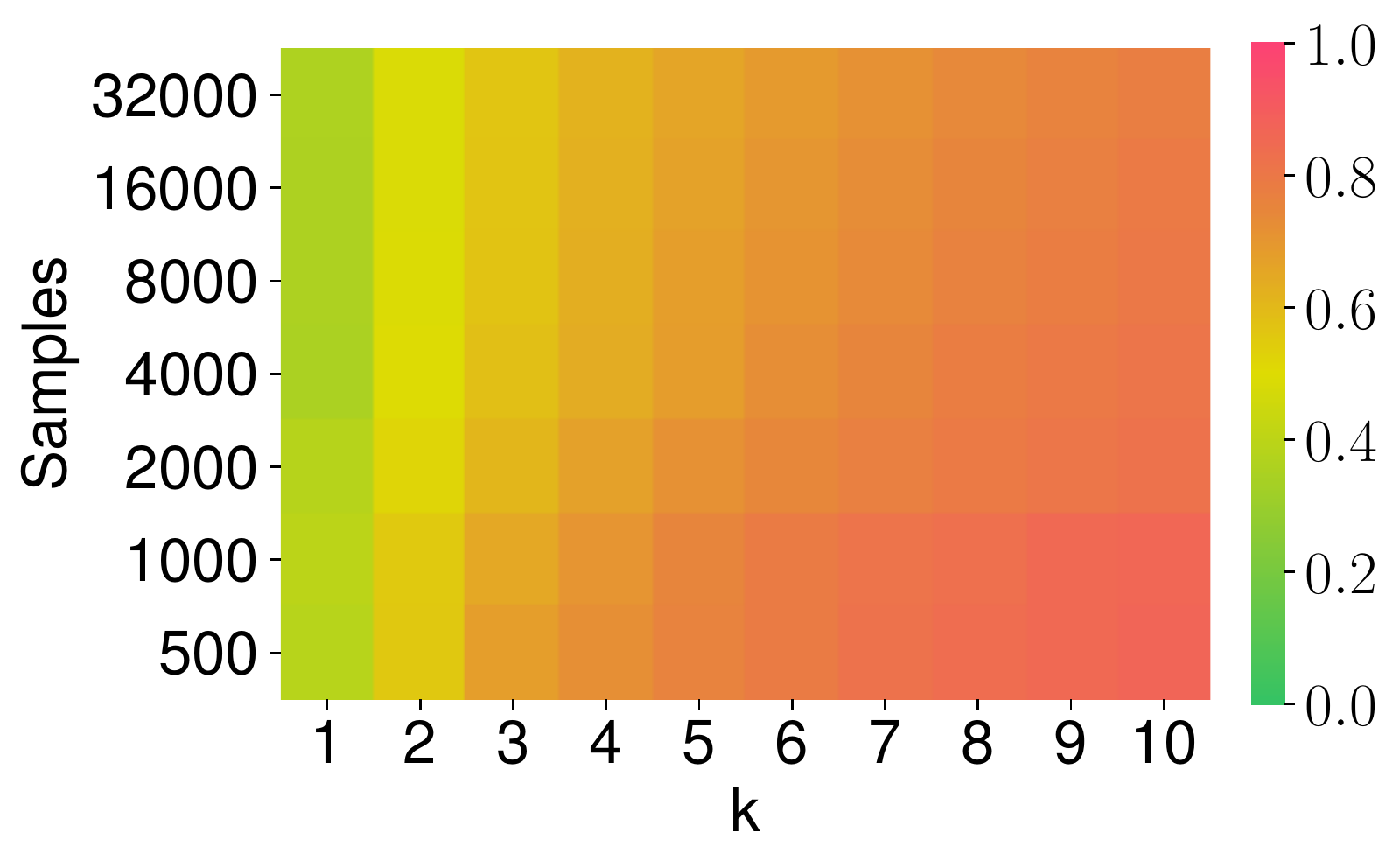} &
        \includegraphics[width=\hyperparamtablesubfloatwidth]{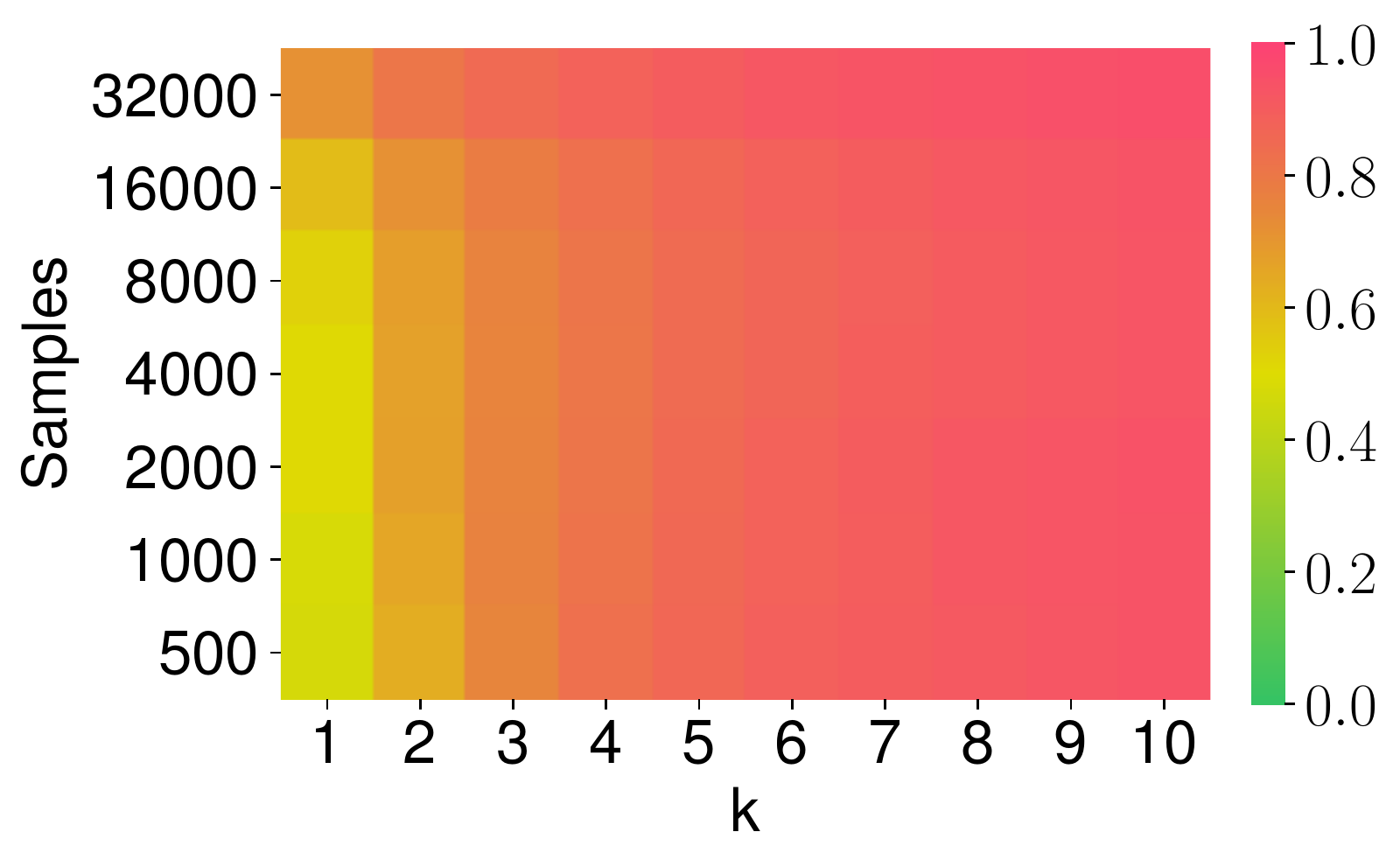} &&
        \includegraphics[width=\hyperparamtablesubfloatwidth]{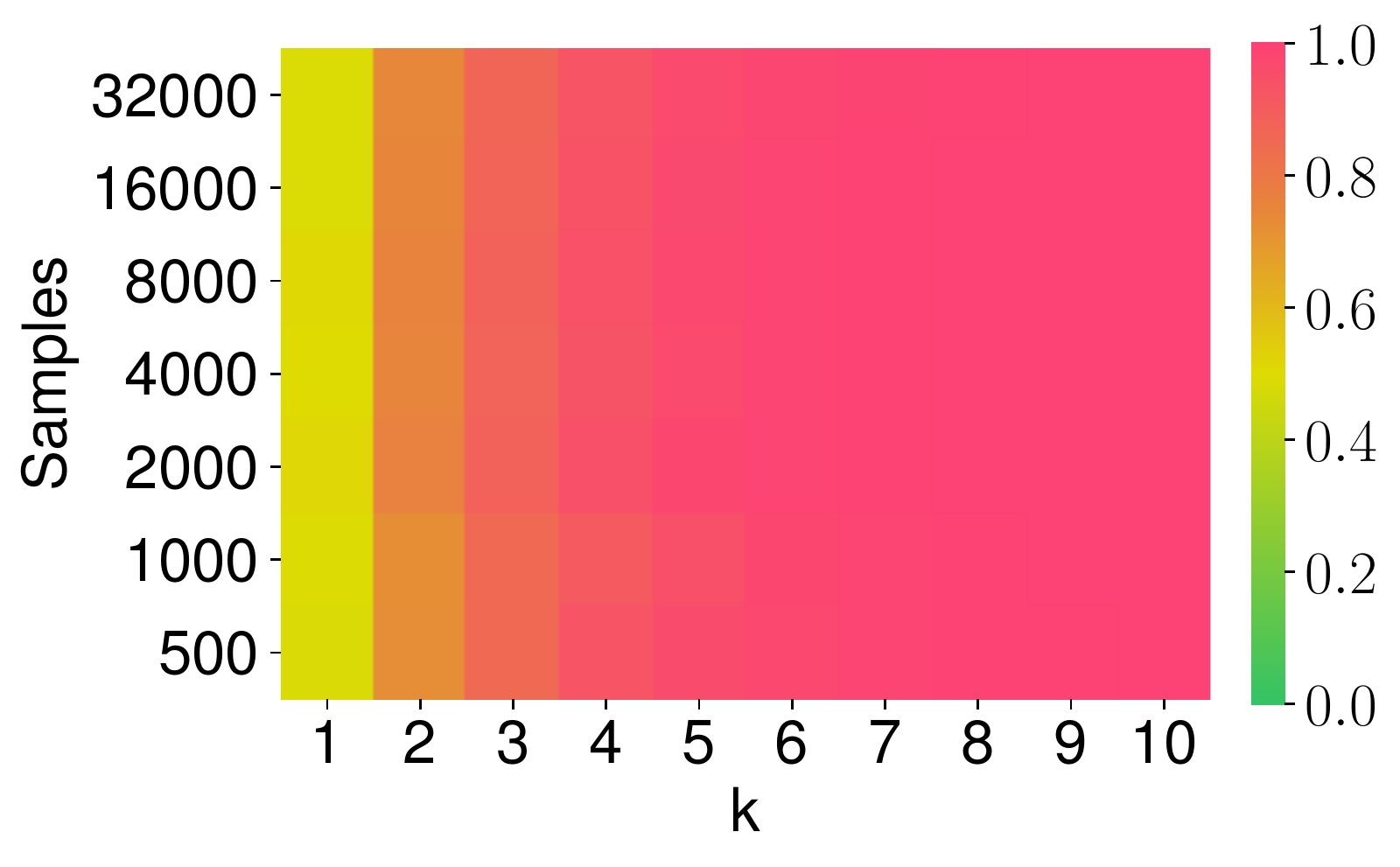} &
        \includegraphics[width=\hyperparamtablesubfloatwidth]{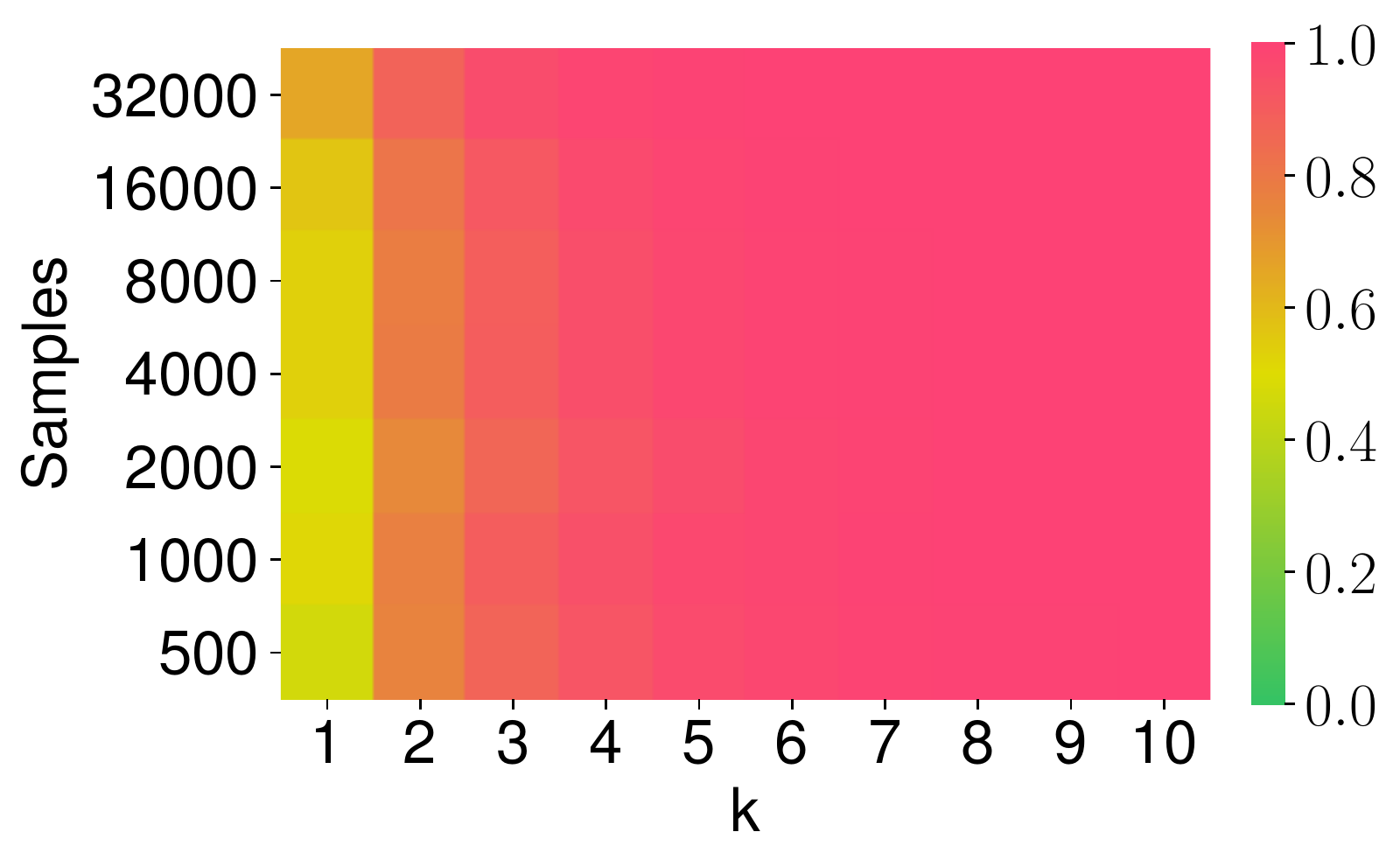} &
        \includegraphics[width=\hyperparamtablesubfloatwidth]{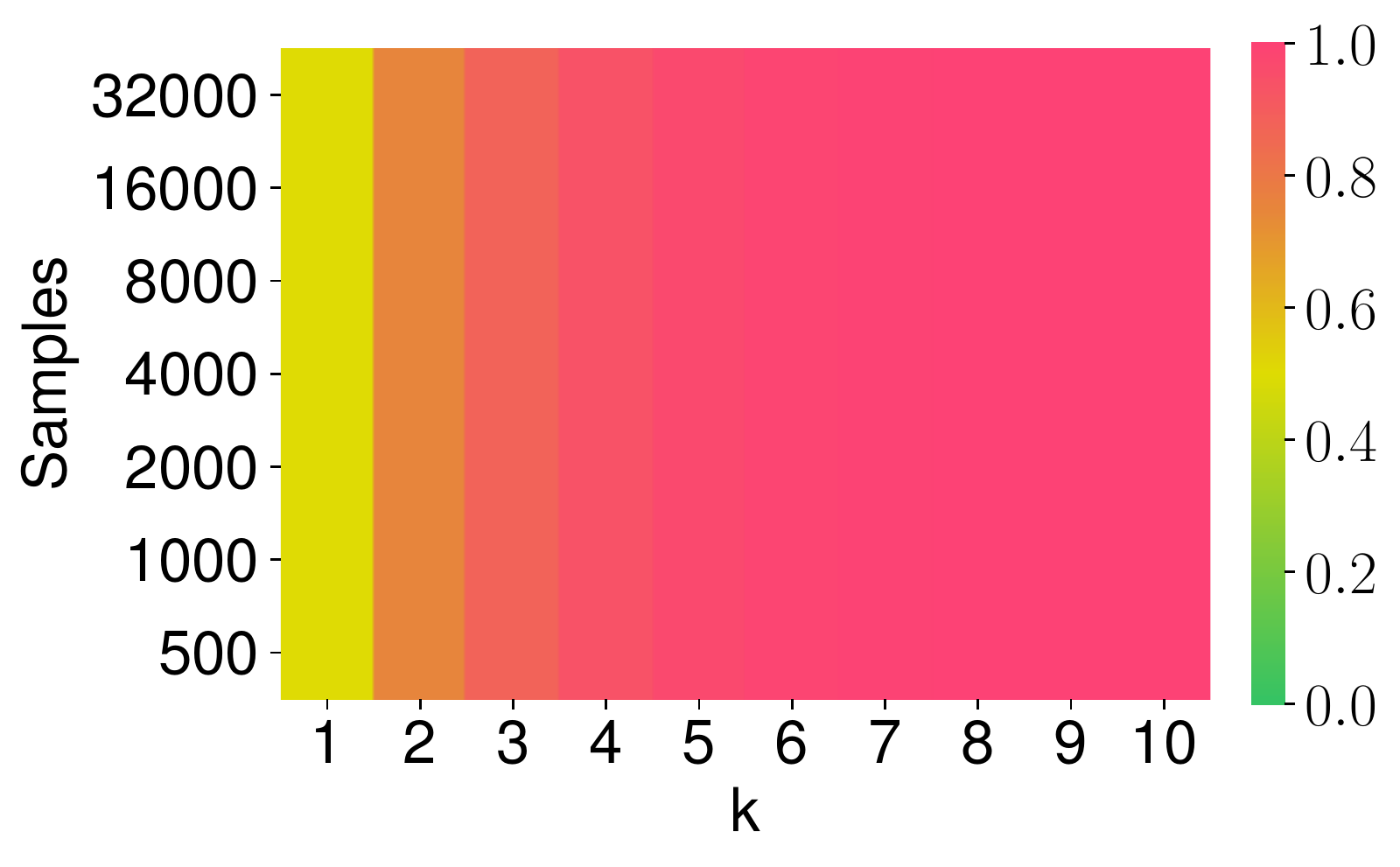} \\
    \end{tabular}
    \vspace{-1em}
    \caption{\small \textbf{Metrics under identical real and fake distributions.} Results with Gaussian and FFHQ data are shown. For density and coverage, the expected values derived in \S\ref{subsec:theory_density_converage} are plotted. In each case, the values are shown with respect to the varying number of samples $N=M$ (assume the same number of real and fake samples) and the nearest neighbours $k$.}
    \label{fig:empirical_metrics_identical_real_fake}
\end{figure*}

\textbf{Analytic derivations for D\&C.}
We derive the expected values of D\&C under the identical real and fake.
\begin{lemma}
$\mathbb{E}[\text{density}]=1$.
\end{lemma}
\begin{proof}
    The expected density boils down to 
    \begin{align}
        \mathbb{E}[\text{density}]
        &=\frac{1}{k}\sum_{i=1}^{N}\mathbb{P}(B^k_N) 
        = \frac{N}{k}\mathbb{P}(B^k_N)
        \nonumber
    \end{align}
    where $B^k_N$ is the event where $\|Y-X_1\|$ is at most $k^{\text{th}}$ smallest among the random variables $\mathbf{S}:=\{\|X_1-X_2\|,\cdots,\|X_1-X_N\|\}$. Since the random variables $\mathbf{S}\cup\{\|Y-X_1\|\}$ are identical and independently distributed, any particular ranking of them are equally likely with probability $1/N!$. Since the number of rankings of $N$ values with one of them at a particular rank is $(N-1)!$, we have
    \begin{align}
    \mathbb{P}(B^k_N) = k\cdot (N-1)! \cdot \frac{1}{N!}=\frac{k}{N}.
    \nonumber
    \end{align}
\end{proof}
\begin{lemma}
    \small
    \begin{align}
    \mathbb{E}[\text{coverage}]
    &=1-\frac{(N-1)\cdots(N-k)}{(M+N-1)\cdots(M+N-k)}.
    \label{eq:expected_coverage}
    \end{align}
    Moreover, as $M=N\rightarrow\infty$, $\mathbb{E}[\text{coverage}]\rightarrow 1-\frac{1}{2^k}$.
\end{lemma}
\begin{proof}
    The expected coverage is estimated as
    \begin{align}
        \mathbb{E}[\text{coverage}]
        &=\frac{1}{N}\sum_{i=1}^N\mathbb{P}(\underbrace{\exists\text{ $j$ s.t. } Y_j\in B(X_i,\text{NND}_k(X_i))}_{(\star)})\nonumber \\
        &=1-\mathbb{P}(\underbrace{\forall\text{ $j$, } Y_j\notin B(X_1,\text{NND}_k(X_1))}_{(\star\star)})
        \nonumber
    \end{align}
    using the fact that the events $(\star)$ is symmetrical with respect to $i$. 
    The event $(\star\star)$ can be re-written as 
    \begin{align}
        \underset{j}{\min}\,\,\|Y_j-X_1\|\geq \|X_1-X_\beta\| \text{ for at most $k$ }\beta\text{'s.}
        \nonumber
    \end{align}
    We write $Z_\beta:=\|X_1-X_\beta\|$ for $\beta\in\{2,\cdots,N\}$ and $\widetilde{Z_j}:=\|X_1-Y_j\|$ for $j\in\{1,\cdots,M\}$. Then, the set $\{Z_\beta\}_{\beta=2}^{N}\cup\{\widetilde{Z_j}\}_{j=1}^{M}$ of $N+M-1$ random variables is independent and identically distributed. The probability of $(\star\star)$ can be equivalently described as:
    \begin{quote}
        Assume there are $M+N-1$ non-negative real numbers $Z$ distributed according to $\overset{\text{iid}}{\sim}\mathbb{P}$. Colour $M$ of them red uniformly at random and colour the rest $N-1$ blue. What is the chance that the $k-1$ smallest among $Z$ are all coloured blue?
    \end{quote}
    Since any assignment of red and blue colours is equally likely, we compute the probability by counting the ratio of possible colour assignments where $k$ smallest elements are coloured blue. The formula is written as
    \begin{align}
    \frac{{{M+N-k-1}\choose{M}}}{{{M+N-1}\choose{M}}}
    =\frac{(N-1)\cdots(N-k)}{(M+N-1)\cdots(M+N-k)}.
    \nonumber
    \end{align}
\end{proof}

Note that the expected values do not depend upon the distribution type or the dimensionality $D$ of the data.

\begin{figure*}[!t]
    \newcommand{\toytablesubfloatwidth}{.3\linewidth}
    \centering
    \begin{tabular}{ccc}
    \includegraphics[width=\toytablesubfloatwidth]{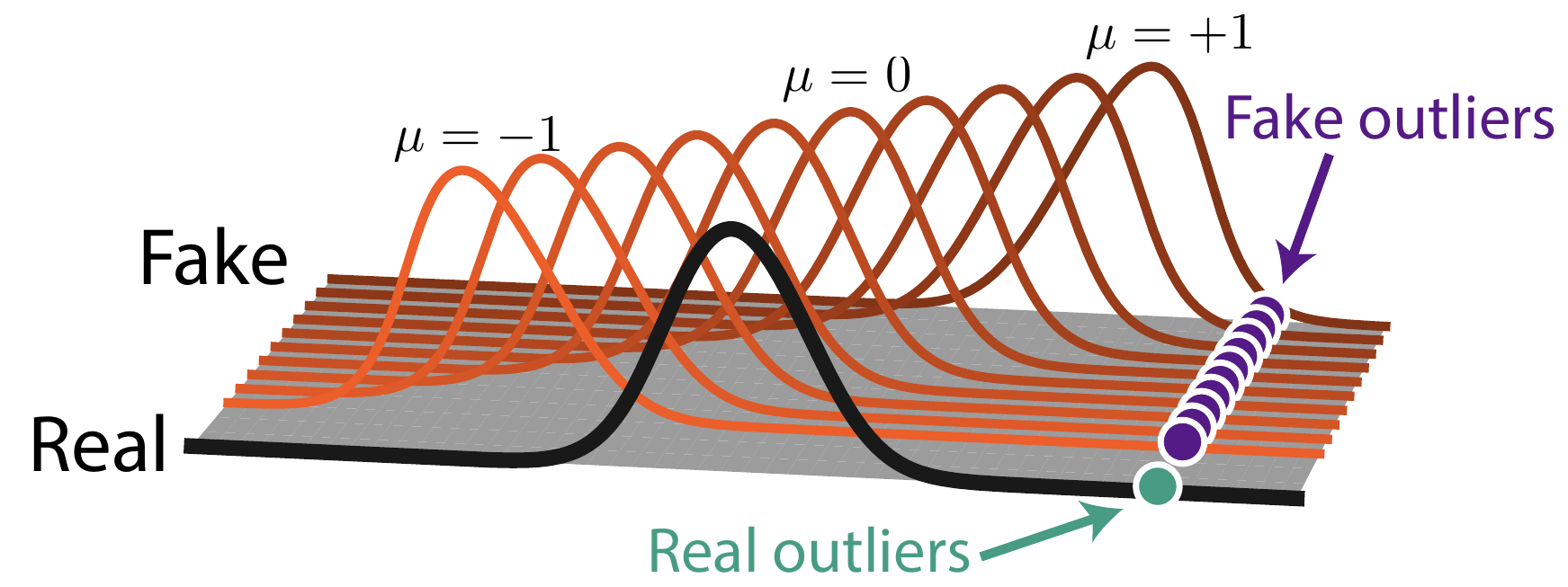} & \includegraphics[width=\toytablesubfloatwidth]{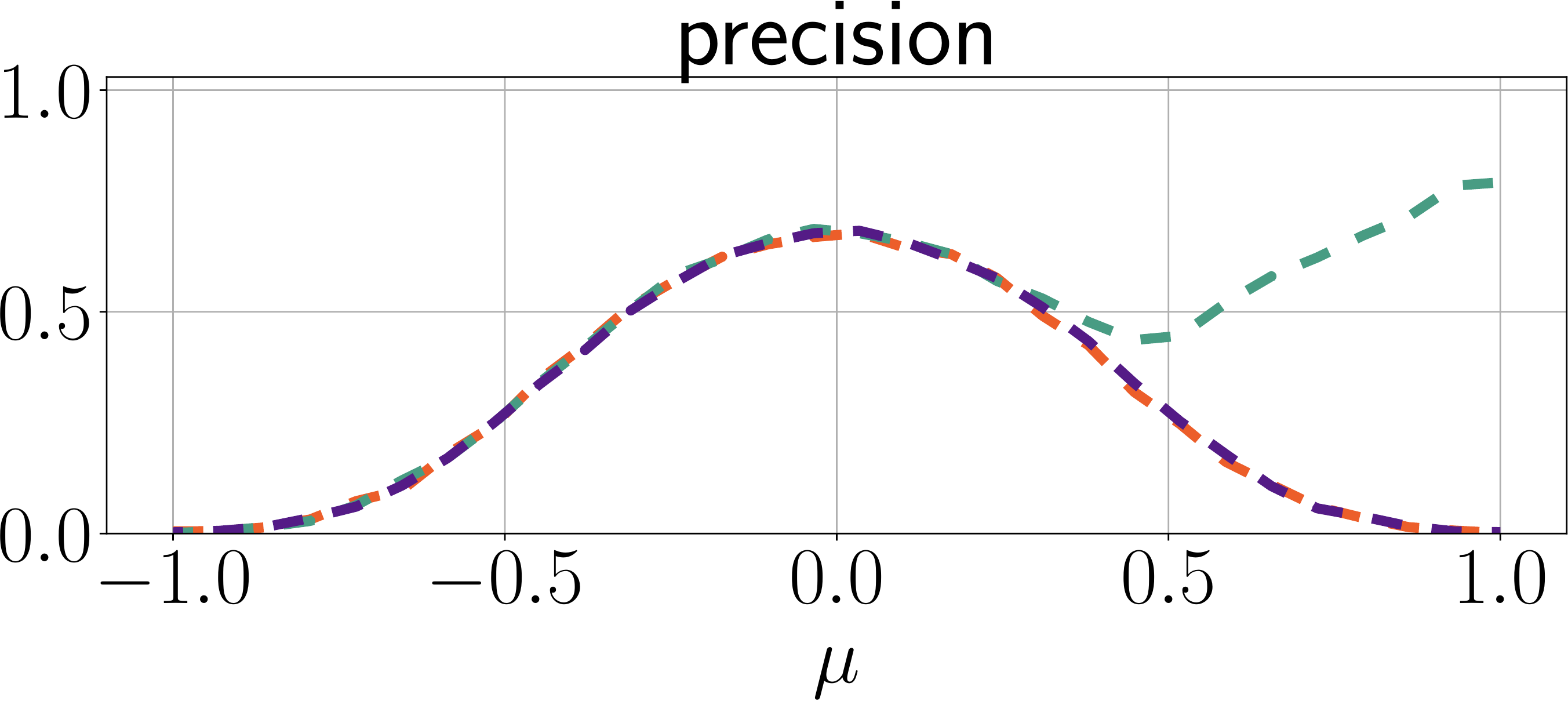} & \includegraphics[width=\toytablesubfloatwidth]{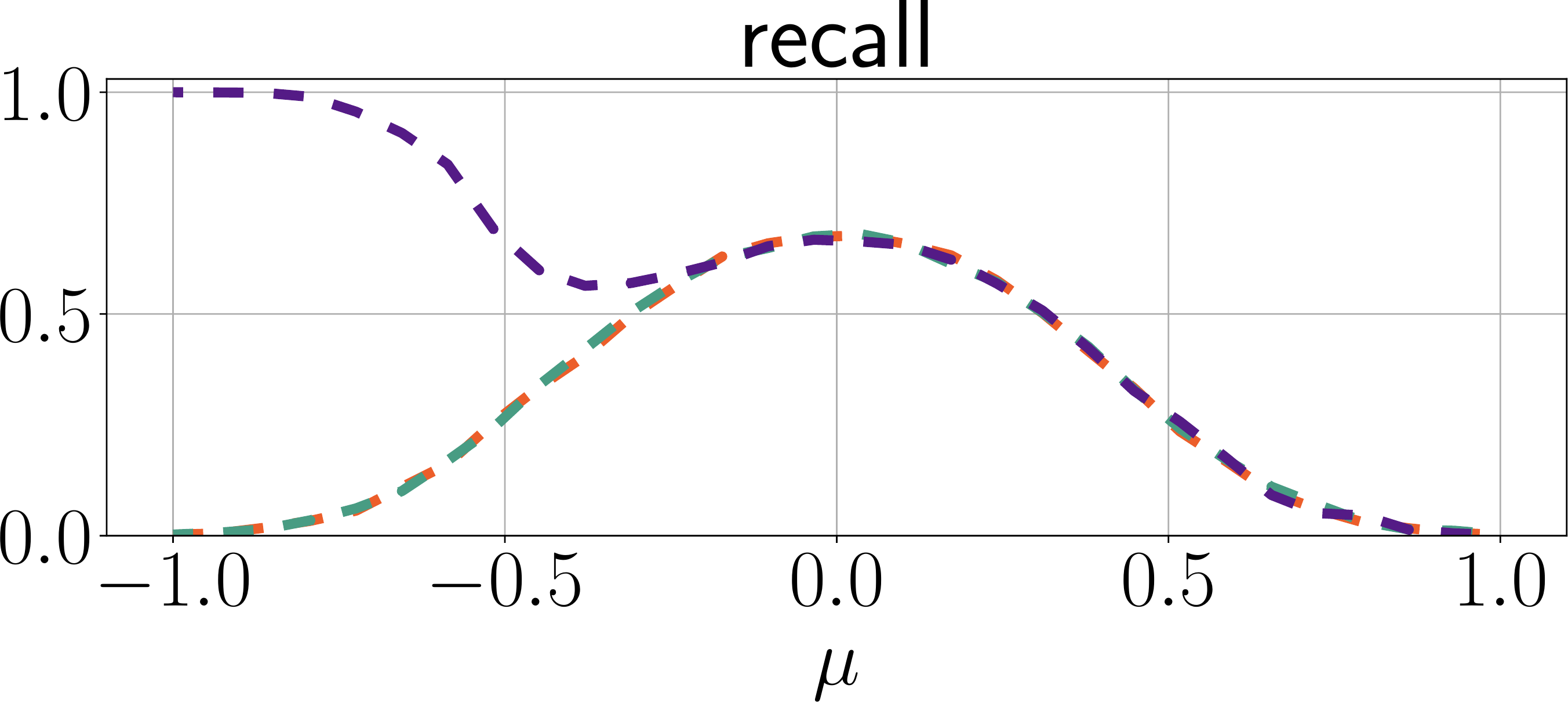}\\
\includegraphics[width=\toytablesubfloatwidth]{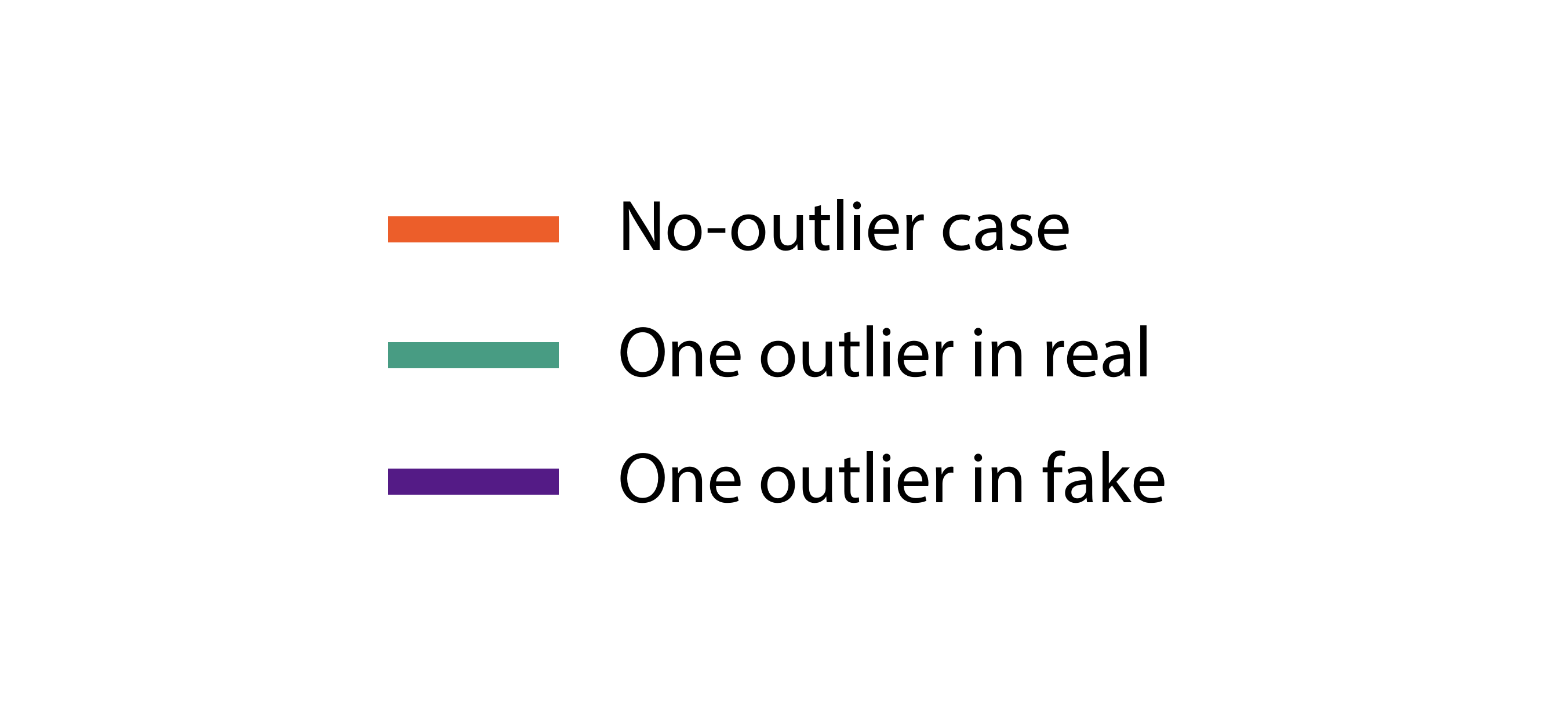} & \includegraphics[width=\toytablesubfloatwidth]{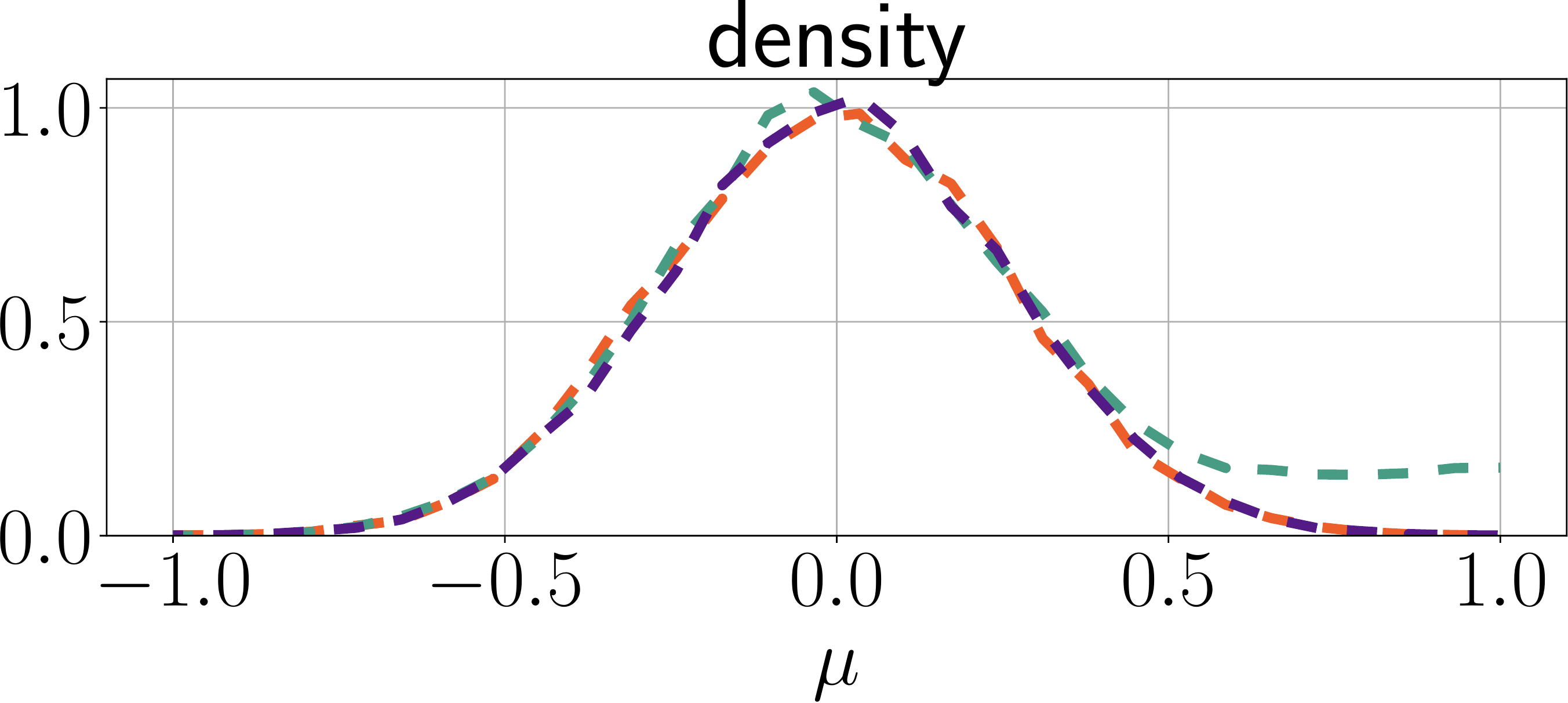} & \includegraphics[width=\toytablesubfloatwidth]{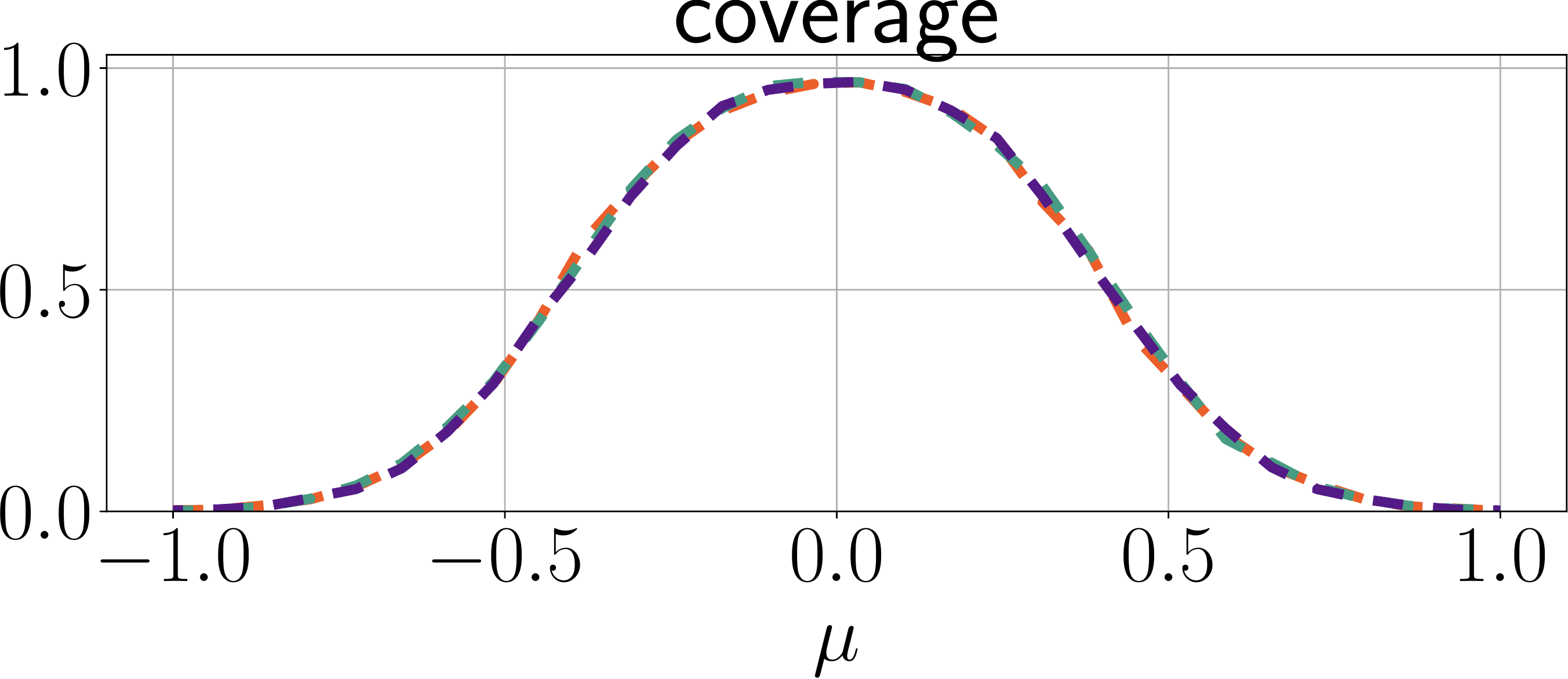}
    \end{tabular}
    \vspace{-1em}
    \caption{\small \textbf{Robustness to outliers for toy data.} Behaviour of the four metrics when the real distribution is fixed $X\sim N(0,I)$ and fake distribution $Y\sim N(\mu,I)$ shifts with $\mu\in [-1,1]$. We further consider two outlier scenarios where a sample at $x=+3$ is added either to the set of real or fake samples.}
\label{fig:fprdc_toy_outliers}
    \vspace{1em}
    \begin{tabular}{c}    \includegraphics[width=0.25\linewidth]{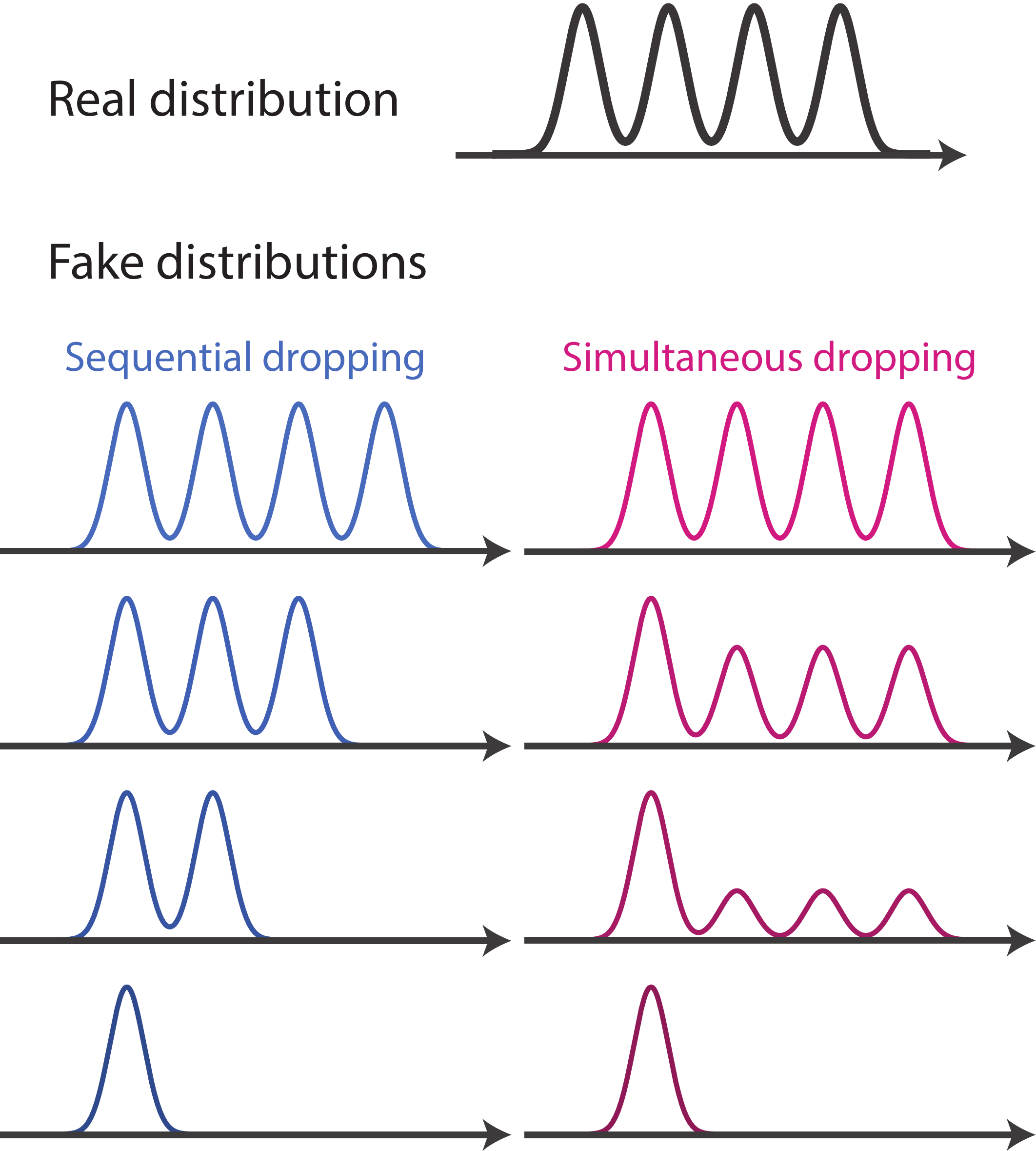}
    \end{tabular}
    \begin{tabular}{cc}
    \includegraphics[width=\toytablesubfloatwidth]{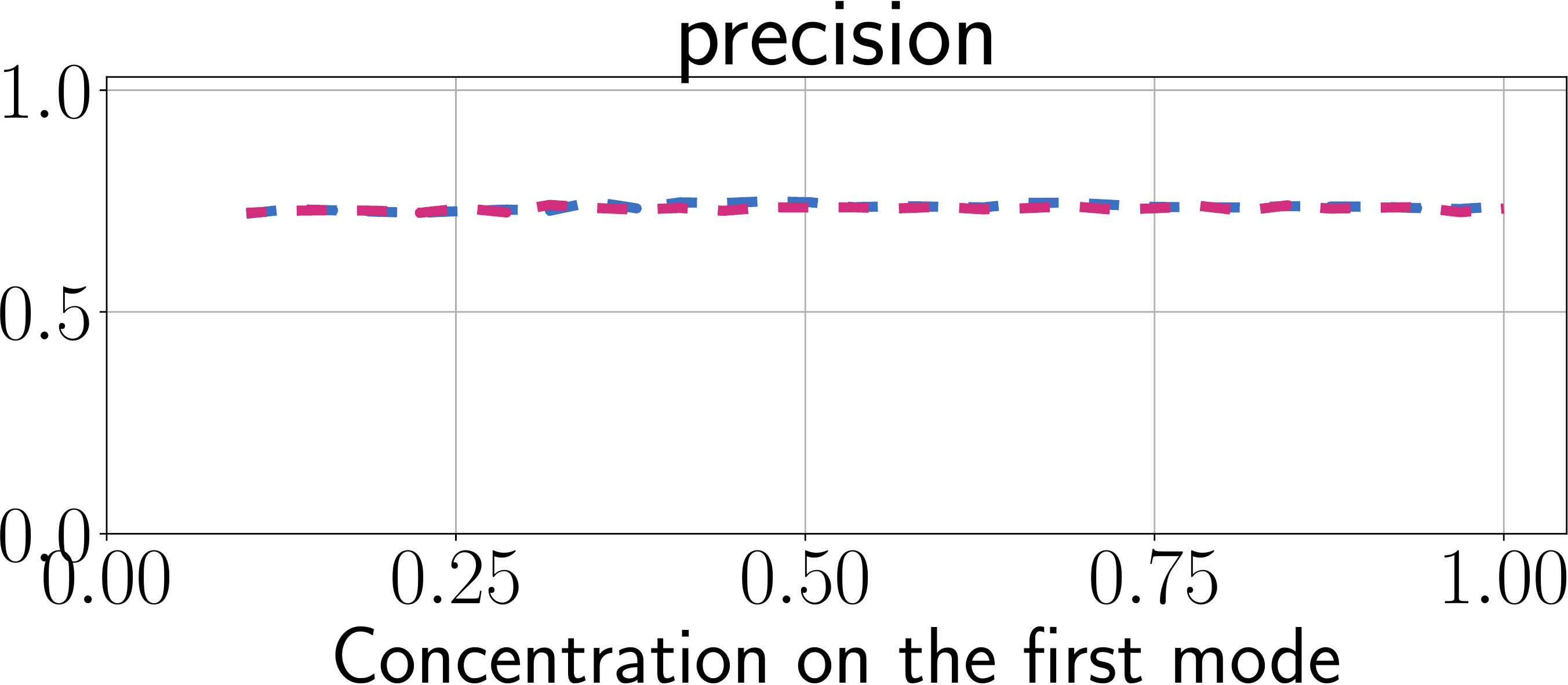} & \includegraphics[width=\toytablesubfloatwidth]{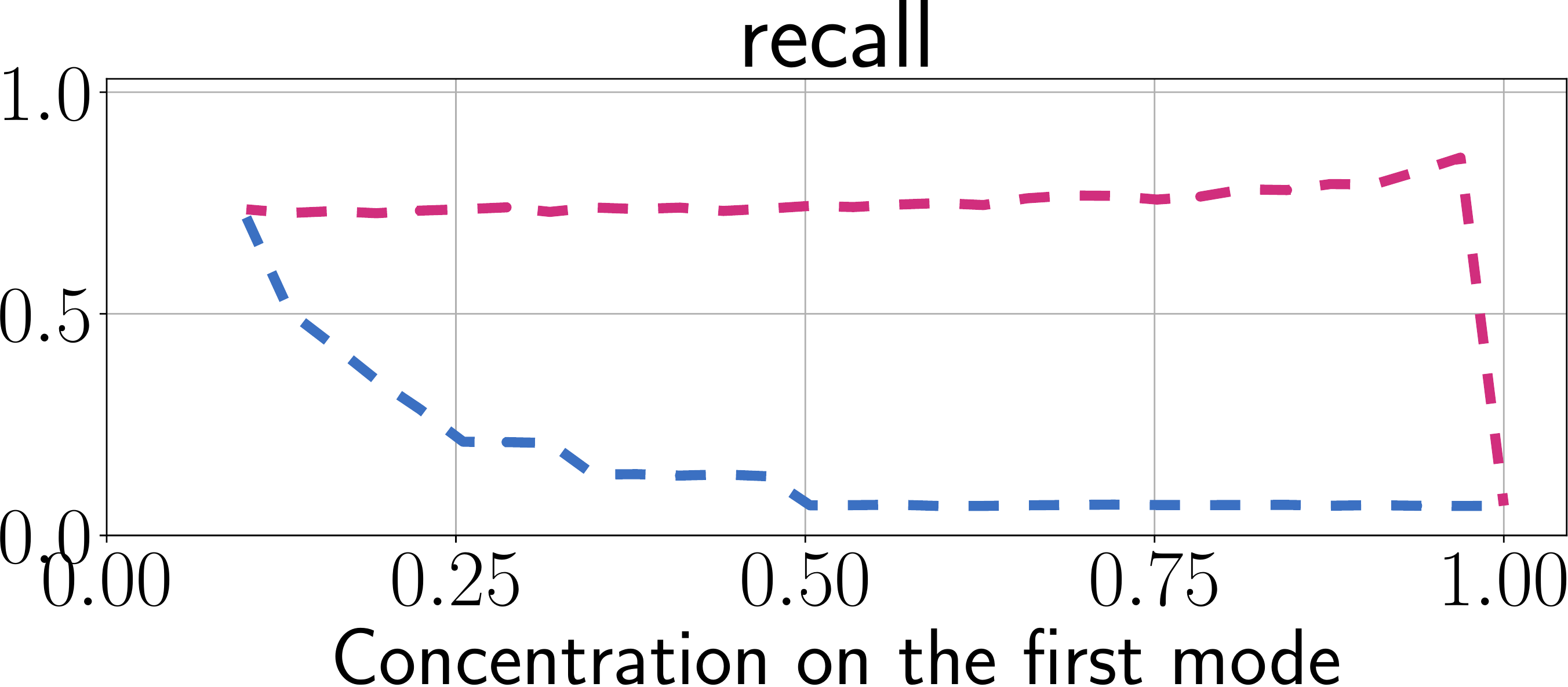}\\
    \includegraphics[width=\toytablesubfloatwidth]{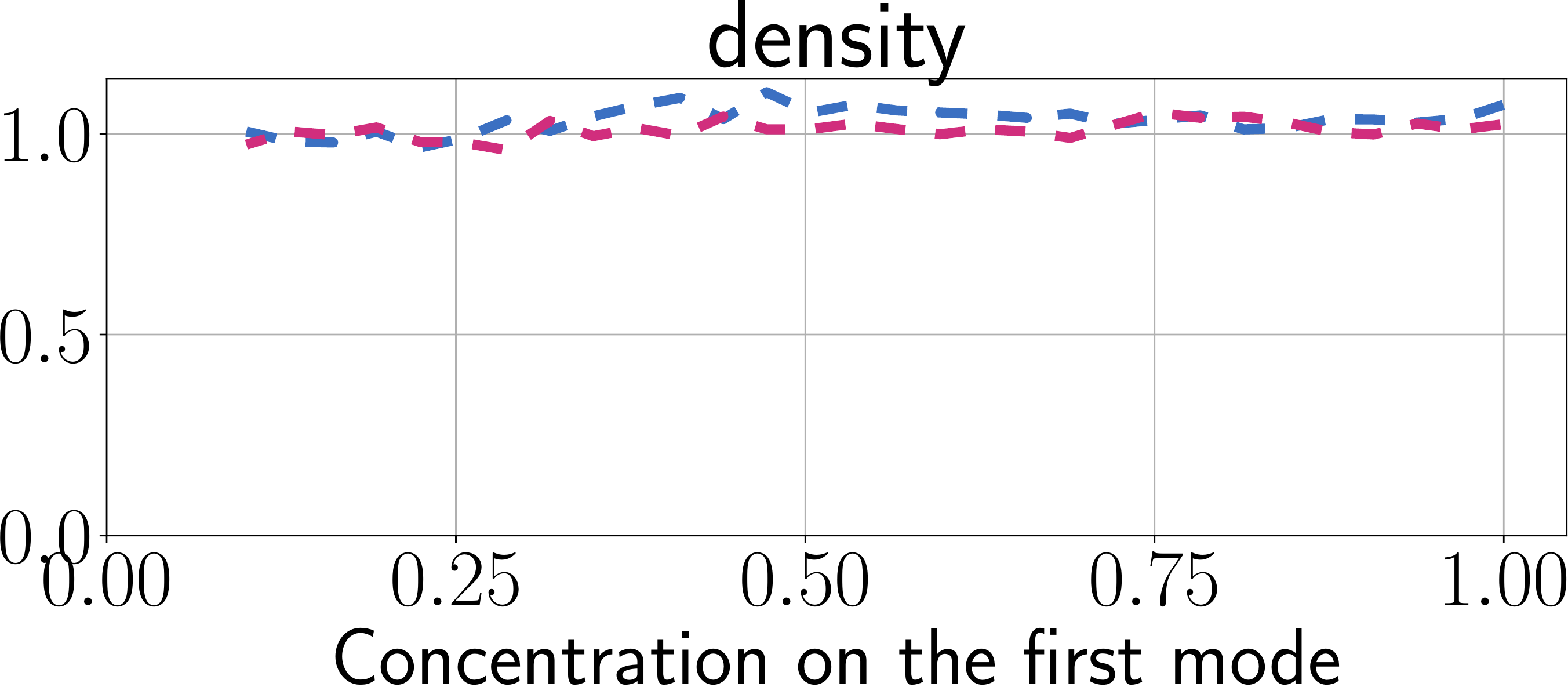} & \includegraphics[width=\toytablesubfloatwidth]{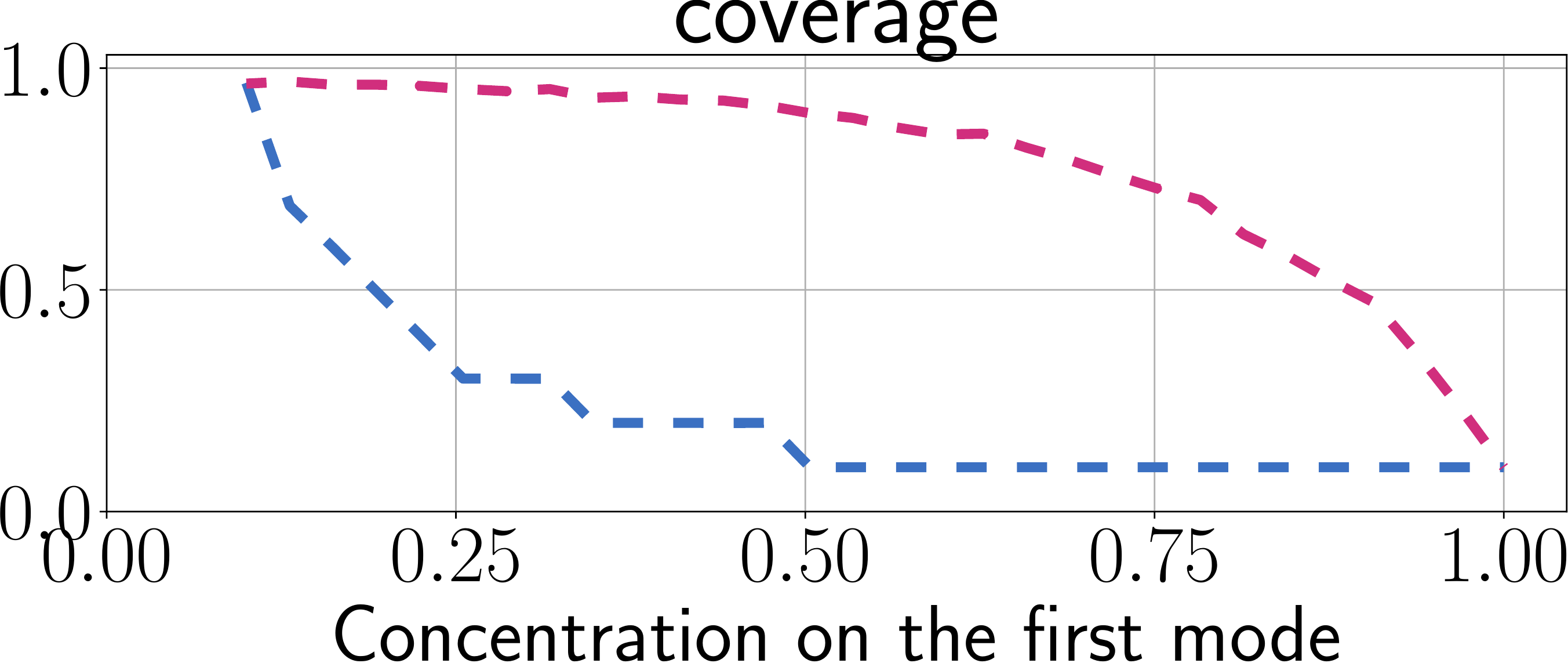}
    \end{tabular}
    \caption{\small \textbf{Sensitivity to mode dropping for toy data.} Assuming the mixture-of-Gaussian real distribution $X$, we simulate two paths in which the perfect fake distribution $Y$ drops to a single mode. Under the \textit{sequential dropping}, $Y$ drops a mode one at a time; under the \textit{simultaneous dropping}, $Y$ gradually decreases samples in all but one mode. Corresponding behaviours of the four metrics are shown.}
\label{fig:fprdc_toy_mode_dropping}
\end{figure*}

\subsection{Hyperparameter selection}

Given the analytic expression for the expected D\&C for identical real and fake distributions, we can systematically choose $k$, $M$, and $N$. We set the aim of hyperparameter selection as: $\mathbb{E}[\text{coverage}]>1-\epsilon$ for identical real and fake. 

Since $\mathbb{E}[\text{coverage}]$ does not depend upon the exact distribution type or dimensionality of data (Equation~\ref{eq:expected_coverage}), the hyperparameters chosen as above will be effective across various data types. We verify the consistency of D\&C across data types in Figure~\ref{fig:empirical_metrics_identical_real_fake}. It contains plots of the P\&R and D\&C values for identical real and fake distributions from (1) 64-dimensional multivariate standard Gaussians, (2) 4096-dimensional ImageNet pre-trained VGG embeddings of the FFHQ face images, and (3) analytic estimations in \S\ref{subsec:theory_density_converage}. While P\&R exhibits significantly different estimated values across the Gaussians and the FFHQ embeddings, D\&C metrics agree on all three types of estimates, confirming the independence on data type. This conceptual advantage leads to a confident choice of evaluation hyperparameters that are effective across data types.

In practice, we choose the hyperparameters to achieve $\mathbb{E}[\text{coverage}]>0.95$. For the sake of symmetry, we first set $M=N$. We then set $M=N=10\,000$ to ensure a good approximation $\mathbb{E}[\text{coverage}]\approx 1-\frac{1}{2^k}$, while keeping the computational cost tractable. $k=5$ is then sufficient to ensure $\mathbb{E}[\text{coverage}]\approx 0.969>0.95$. The exact principle behind the choices of those values for P\&R ($k=3$, $M=N=50\,000$) is unknown.

\section{Experiments}
\label{sec:experiments}

We empirically assess the proposed density and coverage (D\&C) metrics and compare against the improved precision and recall (P\&R)~\cite{ipr2019}. Evaluating evaluation metrics is difficult as the ground truth metric values are often unavailable (which is why the evaluation metrics are proposed in the first place). We carefully select sanity checks with toy and real-world data where the desired behaviours of evaluation metrics are clearly defined. At the end of the section, we study the embedding pipeline (\S\ref{subsec:evaluation_pipeline}) and advocate the use of randomly initialised embedding networks under certain scenarios.

\subsection{Empirical results on density and coverage}
We build several toy and real-world data scenarios to examine the behaviour of the four evaluation metrics: P\&R and D\&C. We first show results on toy data where the desired behaviours of the fidelity and diversity metrics are well-defined (\S\ref{subsec:sanity_artificial}). We then move on to diverse real-world data cases to show that the observations extend to complex data distributions (\S\ref{subsec:sanity_real}). 

\subsubsection{Sanity checks with toy data}
\label{subsec:sanity_artificial}

We assume Gaussian or mixture-of-Gaussian distributions for the real $X$ and fake $Y$ in $\mathbb{R}^{64}$ ($D=64$). We simulate largely two scenarios: (1) $Y$ moves away from $X$ (Figure~\ref{fig:fprdc_toy_outliers}) and (2) $Y$ gradually fails to cover the modes in $X$ (Figure~\ref{fig:fprdc_toy_mode_dropping}). We discuss each result in detail below.

\textbf{Translating $Y$ away from $X$.}
We set $X\sim N(0,I)$ and $Y\sim N(\mu \mathbf{1} ,I)$ in $\mathbb{R}^{64}$ where $\mathbf{1}$ is the vector of ones and $I$ is the identity matrix. We study how the metrics change as $\mu$ varies in $[-1, 1]$. In Figure~\ref{fig:fprdc_toy_outliers}, without any outlier, this setting leads to the decreasing values of P\&R and D\&C as $\mu$ moves away from $0$, the desired behaviour for all metrics. However, P\&R show a pathological behaviour when the distributions match ($\mu=0$): their values are far below 1 (0.68 precision and 0.67 recall). D\&C, on the other hand, achieve values close to 1 (1.06 density and 0.97 coverage). D\&C detect the distributional match better than P\&R.

\textbf{Translating with real or fake outliers.}
We repeat the previous experiment with exactly one outlier at $\mathbf{1}\in\mathbb{R}^{64}$ in either real or fake samples (Figure ~\ref{fig:fprdc_toy_outliers}). Robust metrics must not be affected by one outlier sample. However, P\&R are vulnerable to this one sample; precision increases as $\mu$ grows above $0.5$ and recall increases as $\mu$ decreases below $-0.4$. The behaviour is attributed to the overestimation of the manifold on the region enclosed by the inliers and the outlier (\S\ref{subsec:problem_precision_recall}). This susceptibility is a serious issue in practice because outliers are common in realistic data and generated outputs. On the other hand, the D\&C measures for the outlier cases largely coincide with the no-outlier case. D\&C are sturdy measures.

\textbf{Mode dropping.}
We assume that the real distribution $X$ is a mixture of Gaussians in $\mathbb{R}^{64}$ with ten modes. We simulate the fake distribution $Y$ initially as identical to $X$, and gradually drop all but one mode. See Figure~\ref{fig:fprdc_toy_mode_dropping} for an illustration. We consider two ways the modes are dropped. (1) Each mode is dropped sequentially and (2) weights on all but one mode are decreased simultaneously. Under the sequential mode dropping, both recall and coverage gradually drop, capturing the decrease in diversity. However, under the simultaneous dropping, recall cannot capture the decrease in diversity until the concentration on the first mode reaches 90\% and shows a sudden drop when the fake distribution becomes unimodal. Coverage, on the other hand, decreases gradually even under the simultaneous dropping. It reliably captures the decrease in diversity in this case.

\subsubsection{Sanity checks with real-world data}
\label{subsec:sanity_real}

Having verified the metrics on toy Gaussians, we assess the metrics on real-world images. As in the toy experiments, we focus on the behaviour of metrics under corner cases including outliers and mode dropping. We further examine the behaviour with respect to the latent truncation threshold $\psi$~\cite{karras2019style}. As the embedding network, we use the \texttt{fc2} layer features of the ImageNet pre-trained VGG16~\cite{ipr2019}.

\begin{figure}[!t]
    \small
    \centering
    \setlength{\tabcolsep}{0.3em}
    \begin{tabular}{c|c}
    \textbf{CelebA} & \textbf{LSUN-Bedroom} \\
    \includegraphics[width=.46\linewidth]{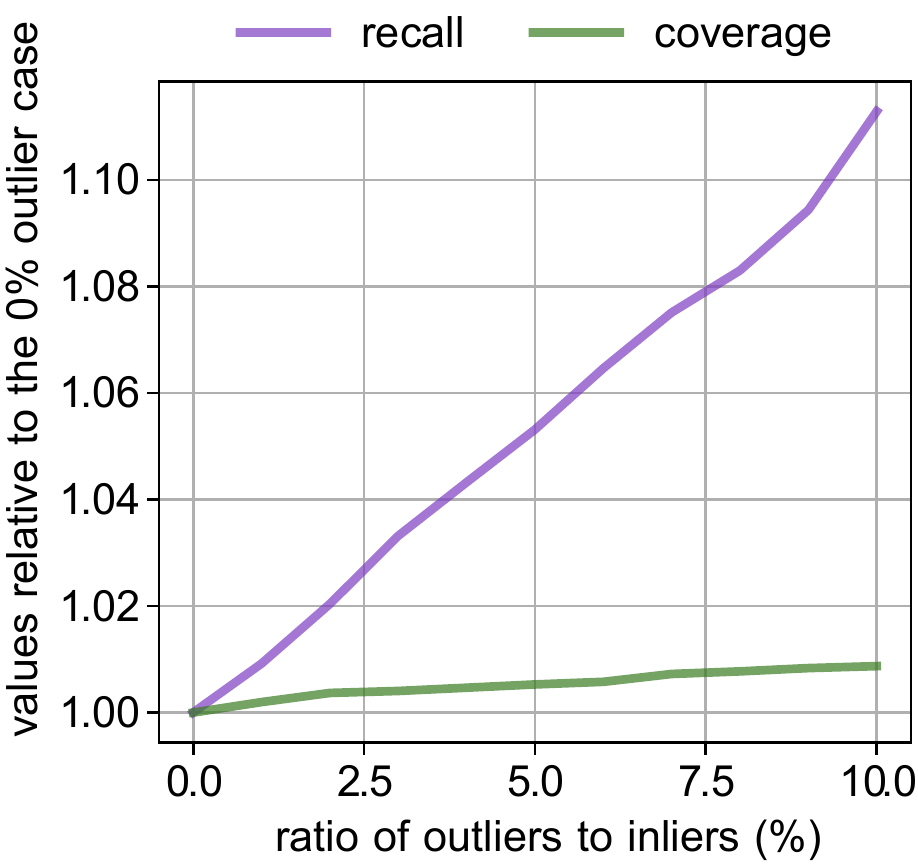} &
    \includegraphics[width=.46\linewidth]{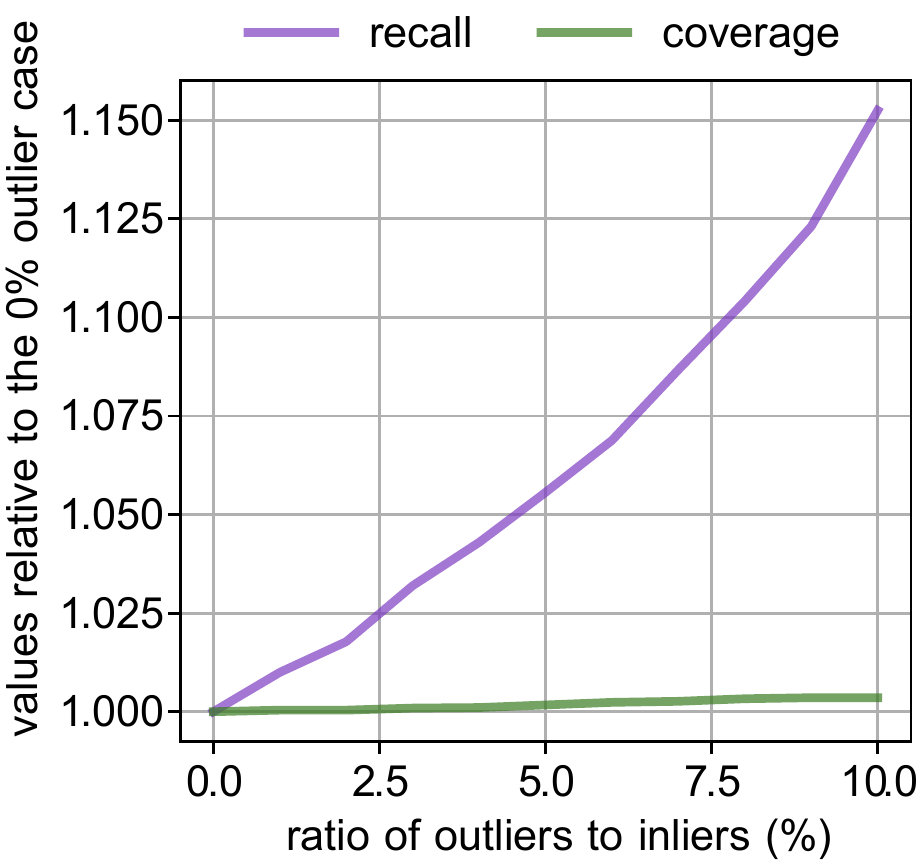} \\
    \vspace{-.5em}& \\
    \includegraphics[width=.45\linewidth]{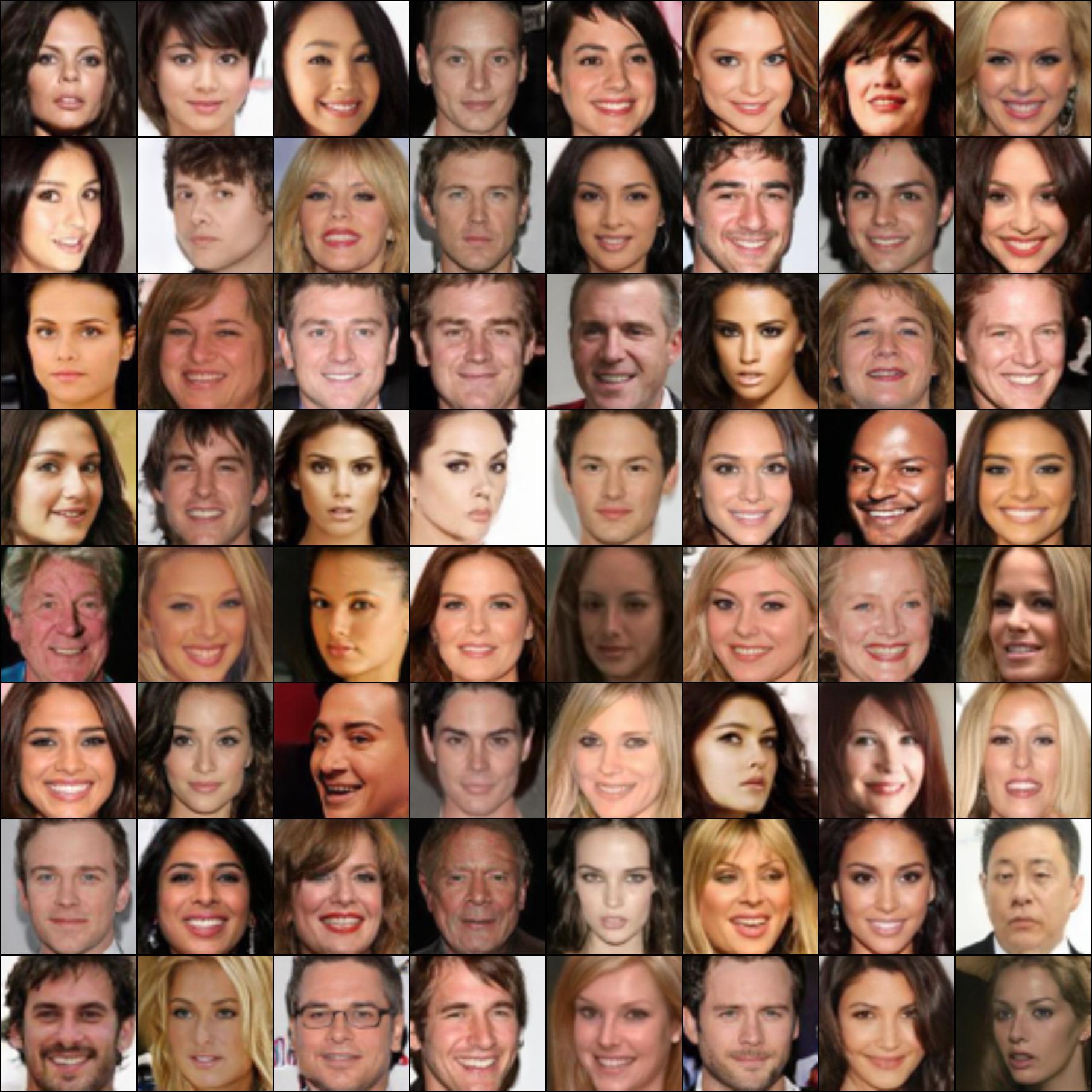} &
    \includegraphics[width=.45\linewidth]{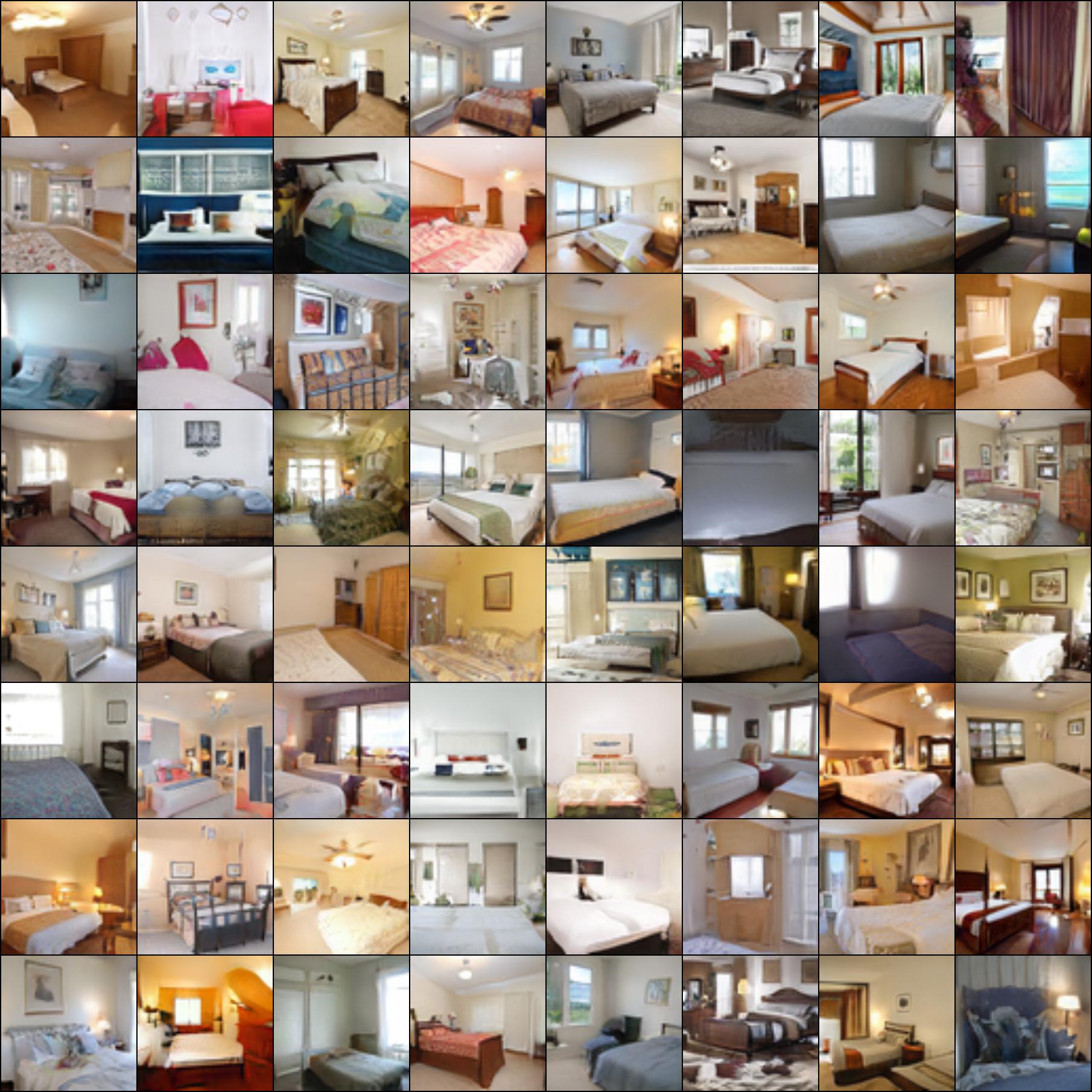} \\
    Inliers & Inliers \\
    \vspace{-.5em}&\\
    \includegraphics[width=.45\linewidth]{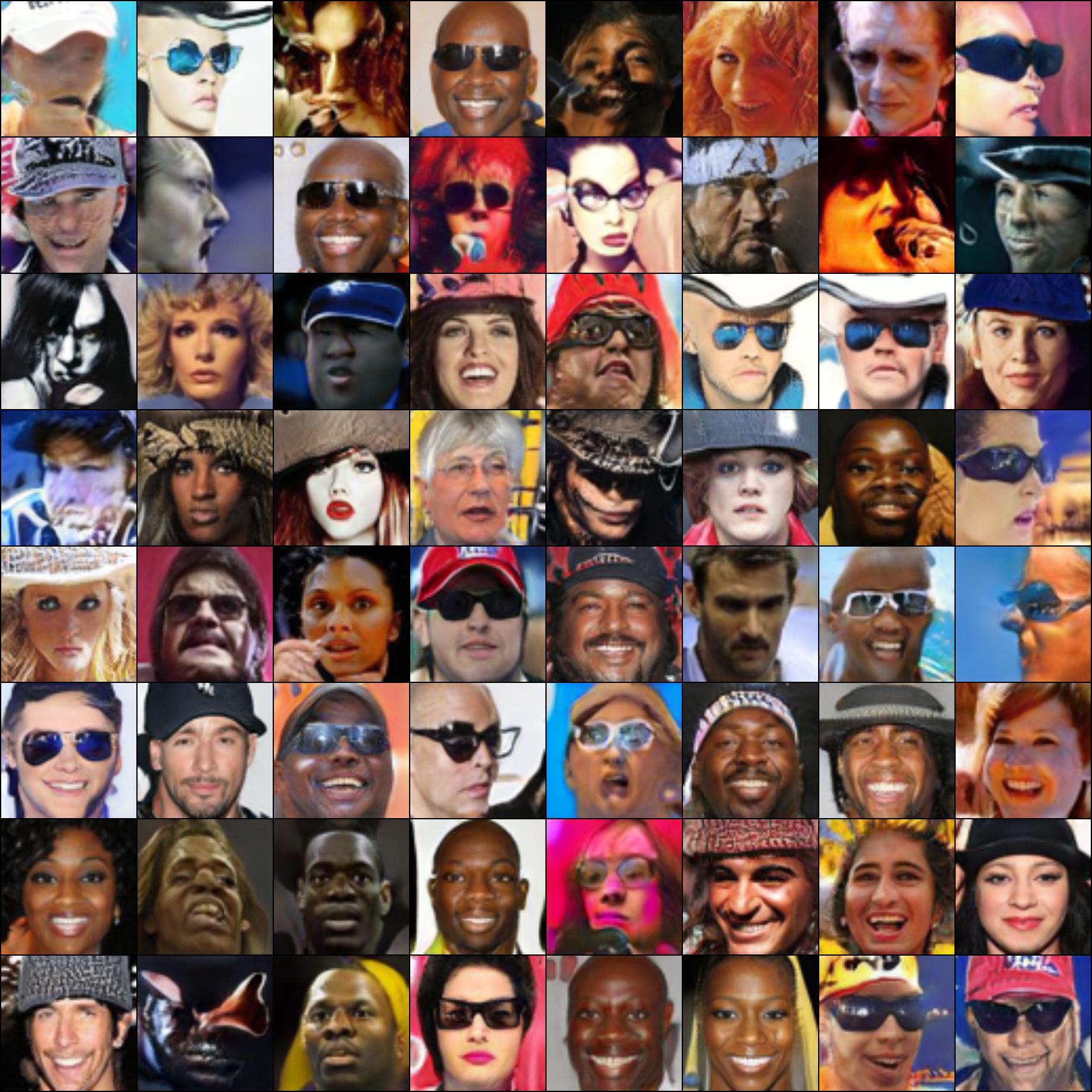} &
    \includegraphics[width=.45\linewidth]{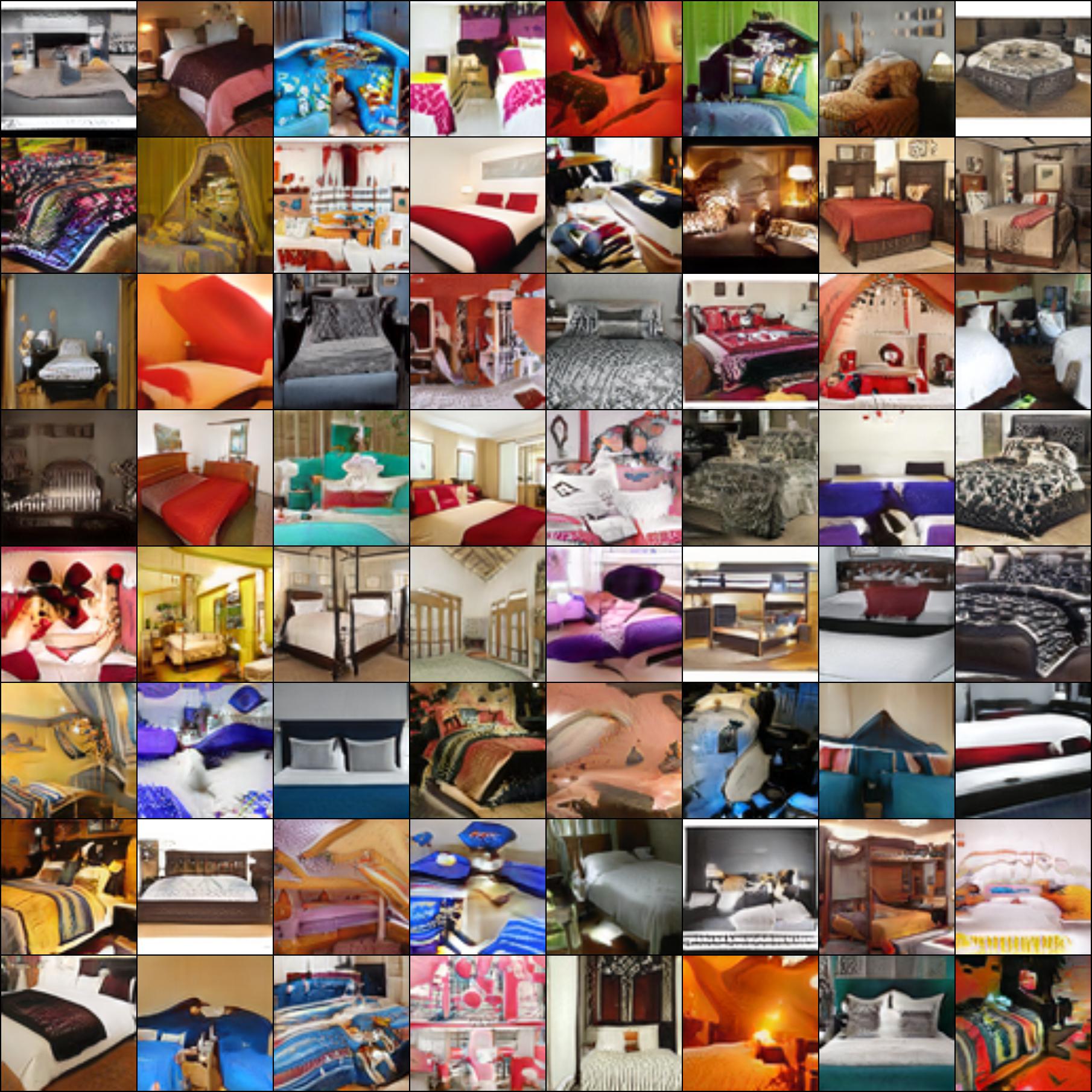} \\
    Worst outliers & Worst outliers
    \end{tabular}
    \caption{\small \textbf{Recall versus coverage against the amount of outliers.
    } 
    Example inliers and the worst outliers are shown below.}
    \label{fig:outlier_generated}
    \subfloat[Precision vs density]{\includegraphics[width=.49\linewidth]{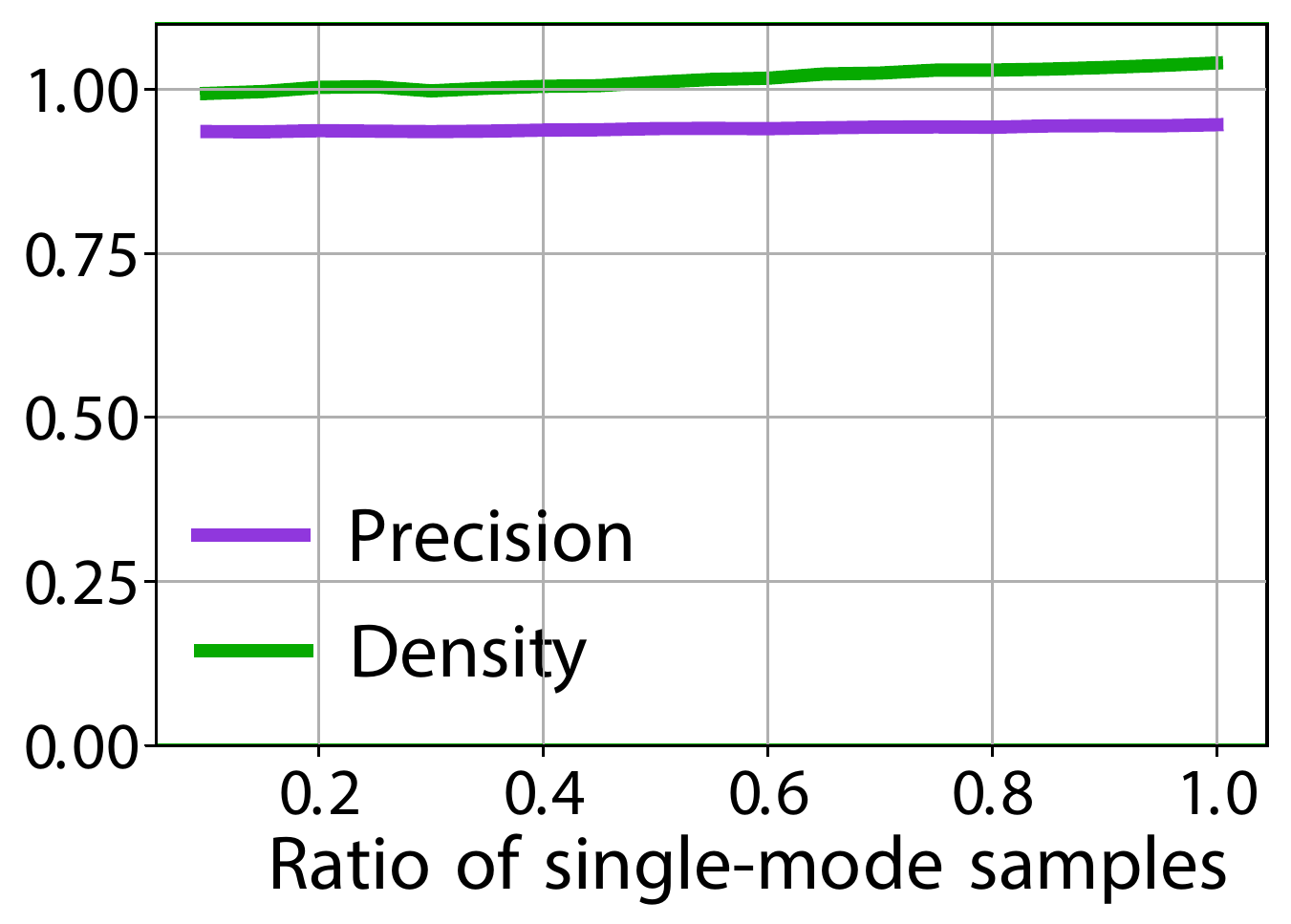}}
    \subfloat[Recall vs coverage]{\includegraphics[width=.49\linewidth]{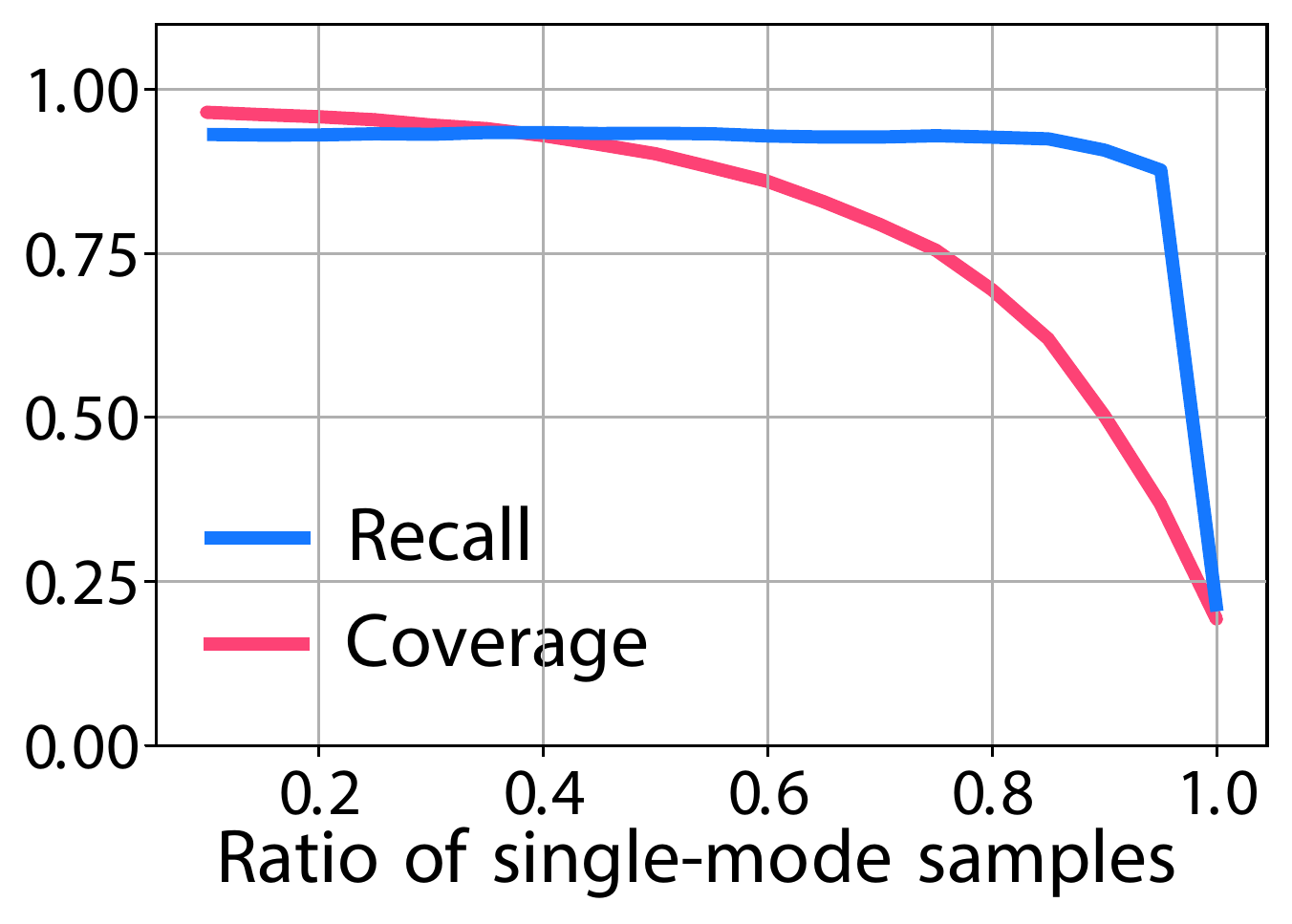}}
    \caption{\small \textbf{Metrics under mode dropping on MNIST.} 
    Behaviour of the metrics under varying degrees of the mode dropping towards a single class on MNIST with \randomEmb. The real distribution has the uniform class distribution while the fake distribution has an increasing ratio of the class ``0''.}
    \label{fig:mnist_mode_dropping}
\end{figure}

\textbf{Outliers.}
Before studying the impact of outliers on the evaluation metrics, we introduce our criterion for outlier detection. Motivated by \cite{ipr2019}, we use the distance to the  $k^\text{th}$ nearest neighbour among the fake samples. According to this criterion, we split the fake samples into $10:1$ for inliers and outliers. We experiment with fake images from StyleGAN on CelebA~\cite{liu2015celeba} and LSUN-bedroom~\cite{yu15lsun}. Example images of inliers and outliers are shown in Figure~\ref{fig:outlier_generated}. We observe that the outliers have a more distortions and atypical semantics. 

We examine the behaviour of recall and coverage as the outliers are gradually added to the pool of fake samples. In Figure~\ref{fig:outlier_generated}, we plot recall and coverage relative to their values when there are only inliers. As outliers are added, recall increases more than 11\% and 15\% on CelebA and LSUN bedroom, respectively, demonstrating its vulnerability to outliers. Coverage, on the other hand, is stable: less than 2\% increase with extra outliers.

\begin{figure}[!t]
    \centering
    \small
    \includegraphics[width=0.45\textwidth]{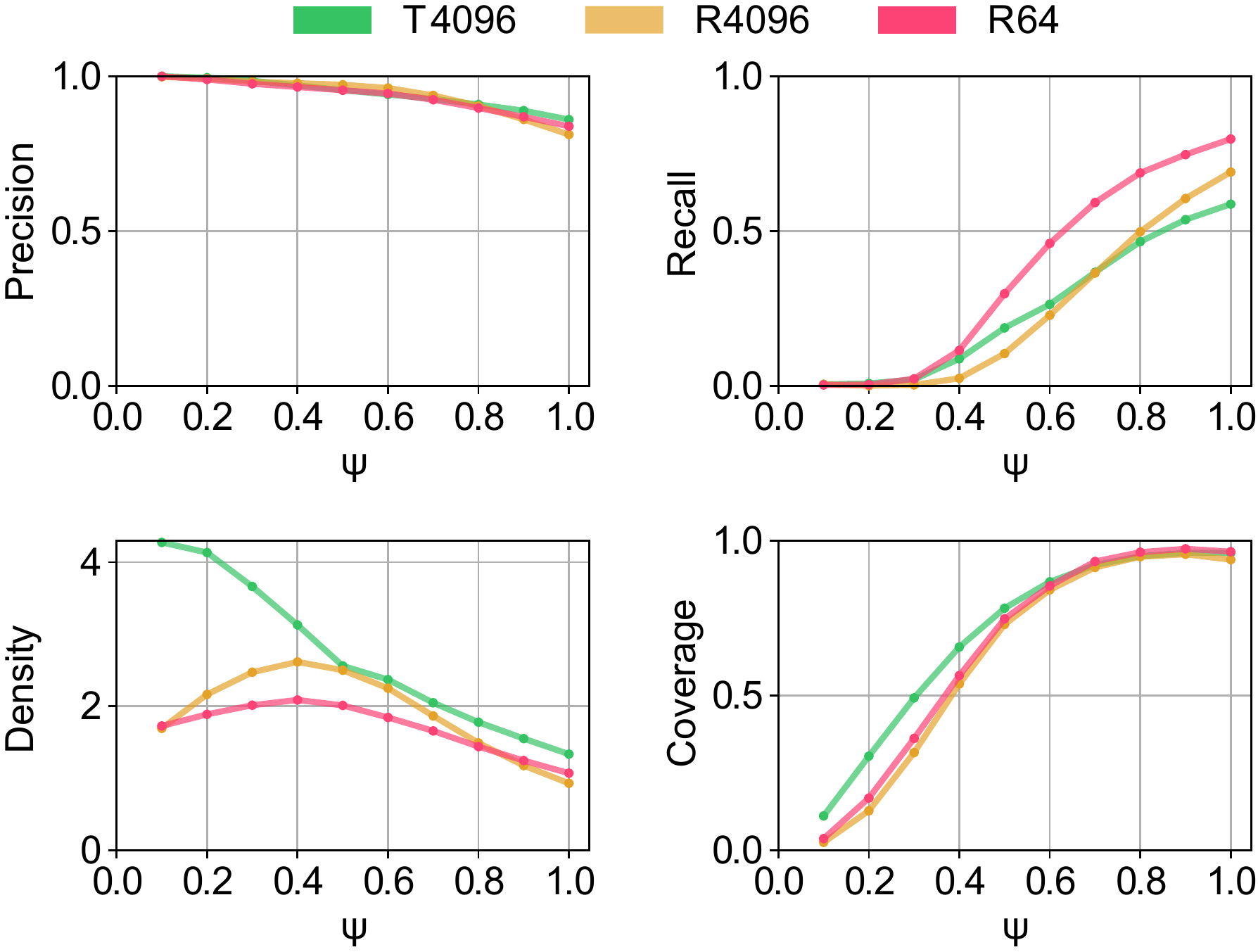}\\
    \vspace{1em}
    \setlength{\tabcolsep}{0.1em}
    \begin{tabular}{ccccc}
    \includegraphics[width=.19\linewidth]{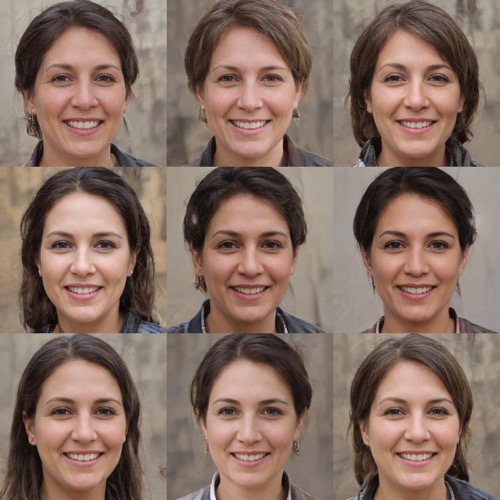}&
    \includegraphics[width=.19\linewidth]{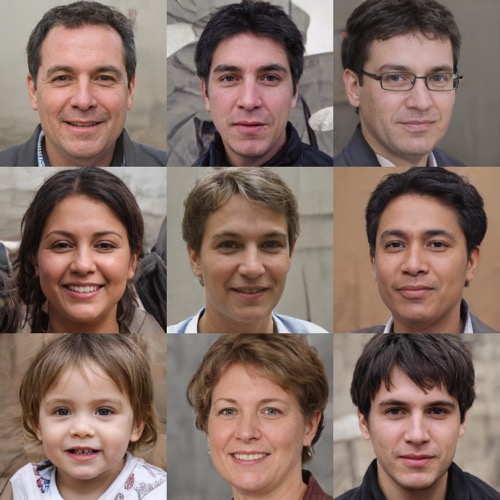}&
    \includegraphics[width=.19\linewidth]{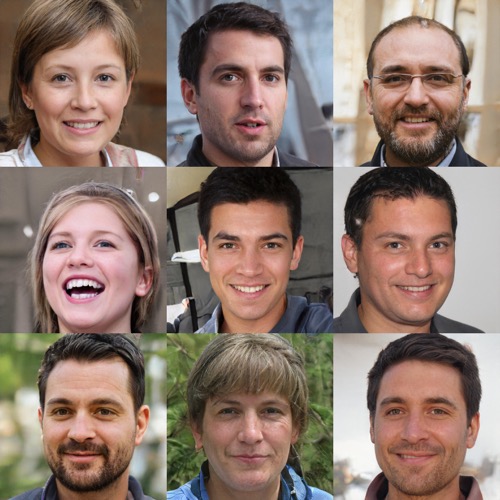}&
    \includegraphics[width=.19\linewidth]{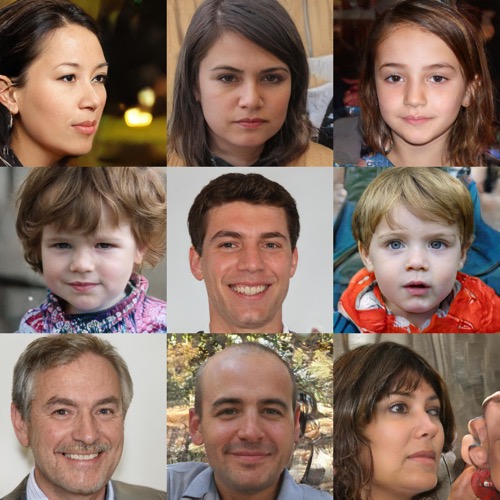}&
    \includegraphics[width=.19\linewidth]{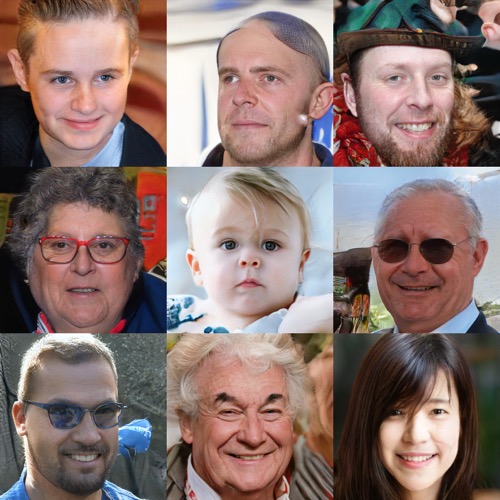}\\
    $\psi=0.2$&
    $\psi=0.4$&
    $\psi=0.6$&
    $\psi=0.8$&
    $\psi=1.0$
    \end{tabular}
    \caption{\small \textbf{Behaviour of metrics with $\psi$.} The latent truncation threshold $\psi$ is applied over over StyleGAN generated images on FFHQ. Qualitative examples are given.}
    \label{fig:embedding_dc_pr_comparison}
\end{figure}

\textbf{Mode dropping.}
As in the toy experiments, we study mode dropping on the MNIST digit images. We treat the ten classes as modes and simulate the scenario where a generative model gradually favours a particular mode (class ``0'') over the others. We use the real data (MNIST images) with decreasing number of classes covered as our fake samples (sequential dropping). The results are shown in Figure~\ref{fig:mnist_mode_dropping}. While recall is unable to detect the decrease in overall diversity until the ratio of class ``0'' is over 90\%, coverage exhibits a gradual drop. Coverage is superior to recall at detecting mode dropping.

\textbf{Resolving fidelity and diversity.}
The main motivation behind two-value metrics like P\&R is the diagnosis involving the fidelity and diversity of generated images. We validate whether D\&C successfully achieve this. We perform the truncation analysis on StyleGAN \cite{karras2019style} on FFHQ. The truncation technique is used in generative models to artificially manipulate the learned distributions by thresholding the latent noise $Z$ with $\psi\geq 0$~\cite{brock2018biggan}. In general, for greater $\psi$, data fidelity decreases and diversity increases, and vice versa. This stipulates the desired behaviours of the P\&R and D\&C metrics.

The results are shown in Figure~\ref{fig:embedding_dc_pr_comparison}. With increasing $\psi$, precision and density decrease, while recall and coverage increase. Note that density varies more than precision as $\psi$ increases, leading to finer-grained diagnosis for the fidelity.

\textbf{Relation to negative log likelihood metrics.}
For variational auto-encoders (VAE,~\cite{kingma2013auto}) and their variants, negative log likelihood (NLL) has been a popular metric for quantifying the ability of the trained model to represent and replicate the data distribution. While NLL provides a single view of the generative model, the D\&C metrics provide further diagnostic information. For example, we have trained a convolutional VAE with MNIST. The mean NLL is computed on the reconstructed test data; the D\&C are computed over randomly-initialised VGG features of dimension 64 (\S\ref{subsec:random_embeddings}). To simulate the fidelity-diversity trade-off for VAEs, we consider truncating the latent variable $z$ at magnitude 1. The NLL for vanilla $-\log p(x_{\text{test}}| z)$ versus truncated latent $-\log p(x_{\text{test}}| z, \|z\| \leq 1)$ is 41.4 and 40.2, not providing meaningful granularity on fidelity-diversity trade-off. For the same setup, D\&C provide further details: density has not changed much (0.208$\rightarrow$0.205), while coverage drops significantly (0.411$\rightarrow$0.291) as the result of truncation. Latent variable truncation for our VAE model leads to a drop in diversity, while not improving the fidelity.

\subsection{Random Embeddings}
\label{subsec:random_embeddings}
Image embedding is an important component in generative model evaluation (\S\ref{subsec:evaluation_pipeline}). Yet, this ingredient is relatively less studied in existing literature. In this section, we explore the limitations of the widely-used ImageNet embeddings when the target data is distinct from the ImageNet samples. In such cases, we propose random embeddings as alternatives. In the following experiments, we write \trainedEmb for the \texttt{fc2} features of VGG16 pre-trained on ImageNet, \randomEmbfull for the \texttt{fc2} features of a randomly initialised VGG16, and \randomEmb for the \texttt{fc2}$^\prime$ features of a randomly initialised VGG16 where the prime indicates the replacement of the 4096 dimensions with 64 dimensions. We experiment on MNIST and a sound dataset, Speech Commands Zero Through Nine (SC09).

\begin{table}[!t]
\setlength{\tabcolsep}{0.1em}
\begin{tabular}{cccc}
    Real MNIST & Real spectrogram \\
    \includegraphics[width=0.45\linewidth]{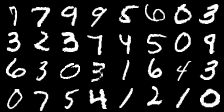} &
    \includegraphics[width=0.45\linewidth]{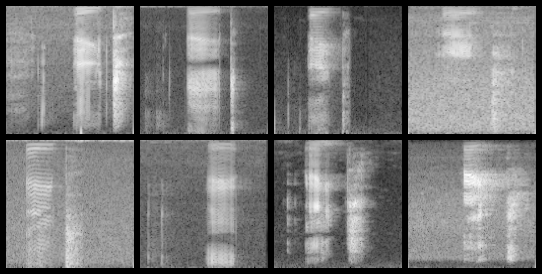} \vspace{0 cm}\\
    DCGAN & WaveGAN \\ 
    \includegraphics[width=0.45\linewidth]{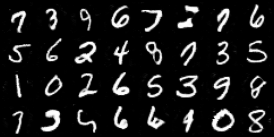} &
    \includegraphics[width=0.45\linewidth]{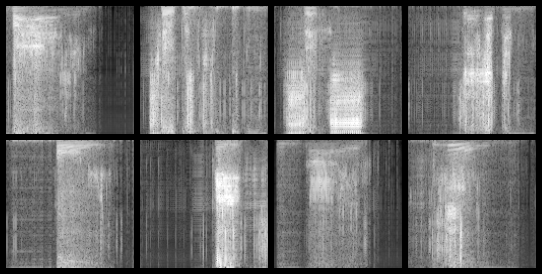}
    \vspace{.5em}
\end{tabular}
\setlength{\tabcolsep}{0.3em}
\centering
\small
\begin{tabular}{rc*{8}{c}}
            && \multicolumn{3}{c}{MNIST} && \multicolumn{3}{c}{Sound}\\
Methods     && \trainedEmb & \randomEmbfull & \randomEmb && \trainedEmb & \randomEmbfull & \randomEmb\\
\cline{1-1} \cline{3-5} \cline{7-9}\vspace{-1em} & \\
Precision   && 0.07 & 0.60 & 0.72 && 0.64 & 0.71 & 0.78\\
Recall      && 0.08 & 0.38 & 0.58 && 0.03 & 0.46 & 0.54\\
Density     && 0.02 & 0.31 & 0.53 && 0.27 & 0.46 & 0.62\\
Coverage    && 0.02 & 0.56 & 0.70 && 0.02 & 0.47 & 0.59\\
\cline{1-1} \cline{3-5} \cline{7-9}
\end{tabular}
\caption{\small \textbf{Failure of trained embeddings.} Evaluation results of generated samples beyond natural images using \trainedEmb, \randomEmbfull, and \randomEmb. Corresponding real and fake qualitative samples shown.}
\label{tab:embedding_trained_failure}
\end{table}

\subsubsection{Metrics with \randomEmb on natural images}
We study the behaviour of metrics for different truncation thresholds ($\psi$) on StyleGAN generated FFHQ images (Figure~\ref{fig:embedding_dc_pr_comparison}). We confirm that the metrics on \randomEmb closely follow the trend with \trainedEmb. P\&R and D\&C over \randomEmb can be used in ImageNet-like images to capture the fidelity and diversity aspects of model outputs. 

\subsubsection{Metrics with \randomEmb beyond natural images}
We consider the scenario where the target distribution is significantly different from the ImageNet statistics. We use DCGAN generated images~\cite{radford2015unsupervised} on MNIST~\cite{lecun1998gradient} and WaveGAN generated spectrograms~\cite{donahue2019wavegan} on Speech Commands Zero Through Nine (SC09) dataset. The qualitative examples and corresponding metrics are reported in Table~\ref{tab:embedding_trained_failure}. 

Compared to the high quality of generated MNIST samples, the metrics on \trainedEmb are generally low (\eg 0.047 density), while the metrics on \randomEmb report reasonably high scores (\eg 0.491 density). For sound data, likewise, the metrics on \trainedEmb do not faithfully represent the general fidelity and diversity of sound samples. The real and fake sound samples are provided at \url{http://bit.ly/38DIMAA} and \url{http://bit.ly/2HAm8NB}. The samples consist of human voice samples of digit words ``zero'' to ``nine''. The fake samples indeed lack enough fidelity yet, but they do cover diverse number classes. Thus, the recall (0.029) and coverage (0.020) values under \trainedEmb are severe underestimations of the actual diversity. Under the \randomEmb, the recall and coverage values are in the more sensible range: 0.572 and 0.653, respectively. When the target data domain significantly differ from the embedding training domain, \randomEmb may be a more reasonable choice. 

\section{Conclusion and discussion}

We have systematically studied the existing metrics for evaluating generative models with a particular focus on the fidelity and diversity aspects. While the recent work on the improved precision and recall~\cite{ipr2019} provide good estimates of such aspects, we discover certain failure cases where they are not practical yet: overestimating the manifolds, underestimating the scores when the real and fake distributions are identical, not being robust to outliers, and not detecting certain mode dropping. To remedy the issues, we have proposed the novel density and coverage metrics. Density and coverage have an additional conceptual advantage that they allow a systematic selection of involved hyperparameters. We suggest future researchers to use density and coverage for more stable and reliable diagnosis of their models. On top of this, we analyse the less-studied component of embedding. Prior metrics have mostly relied on ImageNet pre-trained embeddings. We argue through empirical studies that random embeddings are better choices when the target distribution is significantly different from the ImageNet statistics.

\section*{Acknowledgements}

We thank the Clova AI Research team for the support, in particular Sanghyuk Chun and Jungwoo Ha for great discussions and paper reviews. Naver Smart Machine Learning (NSML) platform~\cite{NSML} has been used in the experiments. Kay Choi has advised the design of figures.

{
\small
\bibliography{main}
\bibliographystyle{icml2020}
}

\end{document}